%% file: paper.tex
\documentclass[11pt,draftcls,onecolumn]{IEEEtran}

\usepackage{authblk}
\author[1]{Henry Kenlay}
\author[2]{Dorina Thanou}
\author[1]{Xiaowen Dong}
\affil[1]{University of Oxford}
\affil[2]{EPFL}

\usepackage{amsthm} 
\usepackage{amsmath} 
\usepackage{mathtools} 
\usepackage{amssymb} 
\usepackage{subcaption} 
\usepackage{tikz-network} 
\usepackage{algorithm}
\usepackage{algorithmic}
\usepackage{graphicx}

\newcommand{\G}{\mathcal{G}} 
\newcommand{\V}{\mathcal{V}} 
\newcommand{\A}{\mathbf{A}} 
\newcommand{\D}{\mathbf{D}} 
\newcommand{\E}{\mathbf{E}} 
\newcommand{\lap}{\mathbf{L}} 
\newcommand{\gso}{\mathbf{\Delta}} 
\newcommand{\R}{\mathbb{R}} 
\newcommand{\x}{\mathbf{x}} 
\newcommand{\U}{\mathbf{U}} 
\newcommand{\equaldef}{\overset{\underset{\mathrm{def}}{}}{=}} 

\DeclareMathOperator*{\diag}{diag} 
\DeclareMathOperator*{\argmax}{arg\,max} 
\DeclareMathOperator*{\clamp}{clamp} 
\DeclareMathOperator{\sign}{sign}
\DeclarePairedDelimiter\set{\lbrace}{\rbrace} 
\DeclarePairedDelimiter\norm{\lVert}{\rVert} 
\DeclarePairedDelimiter\abs{\lvert}{\rvert} 

\newtheorem{lemm}{Lemma}
\newtheorem{prop}{Proposition}
\newtheorem{thm}{Theorem}
\newtheorem{corr}{Corollary}
\newtheorem{defn}{Definition}

\usepackage{microtype}
\usepackage{graphicx}
\usepackage{booktabs}

\usepackage{natbib}

\newcommand{\beginsupplement}{
    \setcounter{table}{0}
    \renewcommand{\thetable}{S\arabic{table}}
    \setcounter{figure}{0}
    \renewcommand{\thefigure}{S\arabic{figure}}} 
    
\title{Interpretable Stability Bounds for Spectral Graph Filters}

\begin{document}

\maketitle

\begin{abstract}
Graph-structured data arise in a variety of real-world context ranging from sensor and transportation to biological and social networks. 
As a ubiquitous tool to process graph-structured data, spectral graph filters have been used to solve common tasks such as denoising and anomaly detection, as well as design deep learning architectures such as graph neural networks. Despite being an important tool, there is a lack of theoretical understanding of the stability properties of spectral graph filters, which are important for 
designing robust machine learning models.
In this paper, we study filter stability and provide a novel and interpretable upper bound on the change of filter output, where the bound is expressed in terms of the endpoint degrees of the deleted and newly added edges, as well as the spatial proximity of those edges. This upper bound allows us to reason, in terms of structural properties of the graph, when a spectral graph filter will be stable. We further perform extensive experiments to verify intuition that can be gained from the bound. 
\end{abstract}

\section{Introduction}
\label{sec:introduction}
A graph is a general-purpose data structure that uses edges to model pairwise interactions between entities, which are modelled as the nodes of the graph. Many types of data in the real world reside on graph domains, such as those collected in sensor, biological, and social networks. 
This has sparkled a major interest in recent years in developing machine learning models for graph-structured data \citep{chami2020machine}, leading to the fast-growing fields of graph signal processing \citep{shuman2013emerging,ortega2018graph} and geometric deep learning \citep{bronstein2017geometric}.

Spectral graph filters, generalisation of classical filters to the graph domain via spectral graph theory \cite{chung1997spectral}, are an ubiquitous tool designed to process graph-structured data.
In addition to various signal processing tasks \citep{shuman2013emerging,ortega2018graph}, graph filters are becoming an important tool for machine learning tasks defined on graphs \citep{dong2020graph}. For example, they have been used to define convolution on graphs and design graph neural networks, which lead to state-of-the-art performance in both node and graph classification \citep{bruna14spectral,chebynet,kipf17semi,cayleynet,wu2019simplifying,rossi2020sign,balcilar2020bridging}.

Despite the surge of research proposing new graph-based machine learning models, significantly less attention has been paid to the understanding of theoretical properties, such as stability, of existing models, in particular graph filters.
Informally, a filter is considered to be stable against a perturbation if, after being applied to a signal, it does not lead to large changes in the filter output. In the context of graph-structured data, stability can be defined with respect to perturbation to the signal or the underlying topology. We focus on the latter in this work as graph filters typically act on a graph topology. 

Stability is important mainly for two reasons. First, real-world graph-structured data often come with perturbations, either due to measurement error or inaccuracy in the graph construction.
Second, when these data are used in machine learning tasks, stability of the graph filters is important to  designing learning algorithms that are robust to small perturbations. As a practical example, graph filters are often used to extract spatio-temporal predictive features from fMRI signals realised as signals on a structural brain network. The underlying structural brain network is typically an approximation of the true brain connectivity and therefore the topology will inherently be noisy. Nevertheless, we would desire the predictions, and thus the filtering process, to be robust to the noise inherent in this data.


There has only been a handful of papers considering the stability of spectral graph filters. Among the existing works \citep{levie2019transferability,polynomialstability,gama2020stability} which address this, most provide an upper bound on the change of the filter output. These upper bounds are in terms of the magnitude of perturbation and lack interpretation in terms of how the structure of the graph has changed, for example the degree of the nodes. This limitation hampers the design of strategies that could defend efficiently against potential adversaries. A notable exception in the literature is the recent work of \citet{kenlay2020stability}; however, this work only considers degree preserving edge rewiring which is a stringent assumption that does not cover many perturbations observed in practical scenarios. 

In this work, we provide a novel upper bound for the output change of spectral graph filters under topological perturbation, i.e., edge deletions and additions, of an unweighted and undirected graph. Unlike previous works, our bound is interpretable in the sense that it is expressed in terms of the structural properties of the graph and the perturbation. The bound helps us understand sufficient conditions under which a spectral graph filter will be stable. Specifically, we show that, when edges are deleted and added to a graph to obtain the perturbed graph, the filter will be stable if 1) the endpoints of these edges have sufficiently high degree, and 2) the perturbation is not concentrated spatially around any one node. 
We further verify the intuition behind our theoretical results using synthetic experiments.

Our study has two main contributions. First, to the best of our knowledge, our theoretical analysis is one of the first that provides sufficient conditions for a graph filter to be robust to the perturbation, where the conditions are in terms of the structural properties rather than the magnitude of change.
Second, unlike previous theoretical studies, we perform extensive experiments to validate the intuition gained from the bound. In particular, we examine how the filter output changes for a range of perturbation strategies including random strategies, an adversarial attack strategy and a robust strategy which is derived from insight that the bound provides. Furthermore, we experiment with a range of random graph models and examine how different properties of these graphs, for example the degree distribution, have an effect on the filter stability.
Overall, we believe this study fills an important gap in our understanding of spectral graph filters, and future work based on these ideas can have broad implications for understanding and designing robust graph-based machine learning models that utilise graph filters, most notably a wide range of designs of graph neural networks \citep{balcilar2020bridging}.

\section{Related work}
\label{sec:related}
\subsection{Stability of graph filters}
Stability of graph filters has been mainly studied by characterising the magnitude of perturbation caused by changes to a graph shift operator (GSO) under the operator norm.
One such effort is the work of \citet{levie2019transferability}, where filters are shown to be stable in the Cayley smoothness space, with the output change being linearly bounded. 
The main limitations of this result is that the constant which depends on the filter is not easily interpretable and the bound is only valid for sufficiently small perturbation. 
In a similar vein, \citet{polynomialstability} proves that polynomial graph filters are linearly bounded under changes to the shifted normalised Laplacian matrix 
by applying  Taylor's theorem for matrix functions \citep{deadman2016taylor}. 
We build upon this work by giving a tighter bound for a larger class of filters, and providing theoretical basis for how the magnitude of change in the Laplacian matrix relates to change in the structural properties of the graph such as the degree of the nodes and the distribution of the perturbed edges.

Studying perturbation with respect to operator norm does not provide invariance to relabelling of nodes.  The authors in \citet{gama2020stability} address this issue by proposing the Relative Perturbation Modulo Permutation, that considers all permutations of the perturbed GSO. 
This measure is at least as hard to compute as the graph isomorphism test for which no polynomial time algorithm is known \citep{babai2016graph}. We believe that incorporating permutations may be beneficial in certain cases but not others. For example, the node labelling in polygon meshes is arbitrary and so the distance should be invariant to the labelling. On the other hand, node labelling in social networks corresponds to a user ID and therefore should remain fixed when measuring the change in the graph. 

The work most related to ours is by \citet{kenlay2020stability} as it  provides structurally interpretable bound. 
The authors show that under degree preserving edge rewiring the change of spectral graph filters applied to the augmented adjacency matrix depends on the locality of edge rewiring and the degree of the endpoints of the edges being rewired. However, they only consider a specific type of perturbation, which greatly simplifies but at the same time limits the analysis. We consider more general perturbations in this work and derive a similar result as a special case. 

\subsection{Adversarial attack and defence}
Adversarial attacks are an optimisation based data-driven approach to finding perturbations for which graph-based models are not robust. In particular, it has been shown that the output of graph neural networks can change drastically even under small changes to the input generated from   adversarial examples  \citep{zugner2018adversarial, sun2018adversarial}.  As our bound necessarily covers worst-case scenarios, adversarial attacks provide insight into the tightness of our bound. 

Complimentary to this line of work, adversarial defence is concerned with the design of models robust to adversarial attacks. In the context of adversarial defence, our work is tangentially related to certified robustness, whose goal is to theoretically guarantee that some nodes will remain correctly classified, e.g., in a semi-supervised node classification task, under small but arbitrary perturbation \citep{bojchevski2019certifiable}.  We are instead interested in certifying what kind of perturbation will lead to small changes in the output of a fixed graph filter. 


\section{Preliminaries and problem formulation}
\label{sec:preliminaries}
We define a graph $\G = (\V, \mathcal{E})$ where $\V$ is a set of $n$ vertices and $\mathcal{E}$ a set of edges. 
We write $u \sim v$ if node $u$ is connected to $v$ and $u \not \sim v$ otherwise. By fixing a labelling on the nodes we can encode $\G$ into a binary adjacency matrix $\A \in \set{0, 1}^{n \times n}$. 
The degree $d_u$ of a node $u$ indicates the number of nodes connected to $u$ and we define the degree matrix as $\D= diag (d_1, \ldots, d_n)$. A node $u$ is said to be isolated if it has degree zero. The normalised Laplacian matrix is defined as $\lap = \mathbf{I}_n - \D^{-1/2}\A \D^{-1/2}$ where $\mathbf{I}_n$ is the identity matrix of dimension $n$ and conventionally the entry $\D^{-1/2}_{uu}$ is set to zero if the node $u$ is isolated. The entries of $\lap$ can be explicitly written as 
\begin{equation}
\label{eq:Laplacian}
\lap_{uv} = \begin{cases}
    1  &\textnormal{if }u = v \\
    \frac{-1}{\sqrt{d_ud_v}}  &\textnormal{if }u \sim v \textnormal{ and } u \not = v \\
    0  &\textnormal{otherwise} \\
    \end{cases}.
\end{equation}
The normalised Laplacian matrix is an example of a GSO, a generalisation of the shift operator from classical signal processing which can be used as a building block to construct a graph signal processing framework \citep{ortega2018graph}. The matrix $\lap$ is real and symmetric, and thus has an eigendecomposition $\lap= \U \mathbf{\Lambda} \U^T$ where $\mathbf{\Lambda}=\diag(\lambda_1 \ldots \lambda_n)$ are the eigenvalues such that $0=\lambda_1 \leq \ldots \leq \lambda_n \leq 2$ and $\mathbf{U}$ is the matrix where the columns are the corresponding unit norm eigenvectors. 

We can define a graph signal $x : \V \rightarrow \R$ as an assignment of each node to a scalar value; this can be compactly represented by a vector $\x$ such that $\x_i = x(i)$. The graph Fourier transform can be defined as $\hat \x = \U^T \x$ and the inverse graph Fourier transform is then given by $\x = \U \hat \x$. With a notion of frequency, filtering signals on graphs amounts to amplifying and attenuating the frequency components in the graph Fourier domain, i.e., $y = \U \diag \large( g(\lambda_1), \ldots, g(\lambda_n) \large) \U^T x = \U g(\Lambda) \U^T \x = g(\lap) \x$, where $g(\cdot)$ is a function over the range of eigenvalues that corresponds to the characteristics of the filter. 

We are primarily concerned with the stability of spectral graph filters where the filter parameters are fixed. This scenario is relevant to hand-tuned filters or during inference of a pre-trained model. An adversarial attack in this setting is known as an evasion attack.
Specifically, we consider edge deletions and additions to the graph to give a perturbed graph $\G_p$ and use $\lap_p$ to denote the normalised Laplacian of $\G_p$. 
We consider the magnitude of the error matrix, i.e., $\norm{\E}_2 = \norm{\lap_p - \lap}_2$, which we call the error norm where $\norm{\cdot}_2$ is the operator norm when applied to matrices and the $\ell_2$-norm when applied to vectors. 
We will also consider the matrix one norm $\norm{\E}_1 = \max_i \sum_j |\E_{ij}|$ and the matrix infinity norm $\norm{\E}_\infty = \max_j \sum_i |\E_{ij}|$. If we denote $\E_u$ as the row corresponding to node $u$ of $\E$ we can write the matrix one norm as $\norm{\E}_1 = \max_u \norm{\E_u}_1$ where $\norm{\cdot}_1$ is the Manhattan or $\ell_1$-norm when applied to vectors.

The goal of this study is two-fold: 1) understand how the relative output of a filter changes when we perturb the topology of the underlying graph; 2) what is the impact of the structural properties of the perturbation on filter stability. In particular, the structural properties we consider are the degree of the nodes before and after perturbation and how much the perturbation is concentrated around each node. We address the first goal in Section \ref{sec:linearlystable} and the second in Section \ref{sec:interpretable}. We experimentally validate the insights gained by our bound in Section \ref{sec:experiments}.


\section{Linearly stable filters}
\label{sec:linearlystable}

Our notion of stability is based on relative output distance defined as 
\begin{equation}
\label{eq:rod}
    \frac{\norm{g(\gso)\x - g(\gso_p)\x}_2}{\norm{\x}_2},
\end{equation}
where $g$ is a spectral graph filter, $\x$ is an input graph signal and $\gso$ is the GSO of the graph $\G$ (similarly $\gso_p$ is the GSO of $\G_p$). If we assume $\x$ has unit norm then the above is equivalent to absolute output distance. We can bound this quantity by what we call the filter distance which measures the largest possible relative output change of the filter over non-zero signals:
\begin{align}
\frac{\norm{g(\gso)\x - g(\gso_p)\x}_2}{\norm{\x}_2} \leq \max_{\x\not = 0} \frac{\norm{g(\gso)\x - g(\gso_p)\x}_2}{\norm{\x}_2} \equaldef \norm{g(\gso) - g( \gso_p)}_2. \label{eq:filter_dist}
\end{align}
In \citep{polynomialstability}, the authors bound the filter distance of a graph filter $g$ where $g$ is a polynomial. The bound is given by some constant times the error norm $\norm{\E}_2$, where the constant depends on the filter. When a filter satisfies this property we say it is linearly stable which we define as follows. 
\begin{defn}
A spectral graph filter $g:\R \to \R$ is said to be linearly stable with respect to a type of graph shift operators, if for any graph shift operators $\mathbf{\Delta}$ and $\mathbf{\Delta}_p$ of this type, we have that \begin{equation} \label{eq:linearlystabledefn}
\norm{g(\mathbf{\Delta}) - g(\mathbf{\Delta}_p)}_2 \leq C \norm{\E}_2\end{equation} for some positive constant $C \in \R$. The positive constant $C$ is referred to as the stability constant. 
\end{defn}

\begin{table*}
\centering
\caption{Examples of linearly stable graph filters used for machine learning.}
\scalebox{0.9}{
\begin{tabular}{ccccc} \toprule
Filter & Functional form & GSO & Stability constant $C$ & Use \\ \midrule 
Polynomial filter & $\sum_{k=0}^K \theta_k \lambda^k$ & $\frac{2\lap}{\lambda_{\max}}-\mathbf{I}_n$ & $\sum_{k=1}^K k \abs{\theta_k}$ & Chebynet \citep{chebynet} \\
\vspace{0.1cm} Low-pass filter & $(1+\alpha \lambda)^{-1}$ & $\lap$ & $\alpha$ & Low-pass filtering \citep{ramakrishna2020user} \\ 
Monomial & $\lambda^K$ & $\tilde{\mathbf{D}}^{-1/2}\tilde{\mathbf{A}}\tilde{\mathbf{D}}^{-1/2}$ & $K$ & Simple GCN \citep{wu2019simplifying} \\ 
Identity & $\lambda$ & $\tilde{\mathbf{D}}^{-1/2}\tilde{\mathbf{A}}\tilde{\mathbf{D}}^{-1/2}$ & $1$ & GCN \citep{kipf17semi} \\
\bottomrule
\end{tabular}}
\label{table:lipschitz}
\end{table*}

Two types of filters of particular interest are the polynomial filters, i.e., $g(\lambda)=\sum_{k=0}^K \theta_k \lambda^k$, where $\{\theta_k\}_{k=0}^K$ are the polynomial coefficients, and the low-pass filters, i.e., $g(\lambda)=(1+\alpha \lambda)^{-1}$, where $\alpha>0$ is some constant. Polynomial filters are used in a variety of graph-based machine learning. We list some of them in Table \ref{table:lipschitz}.
It was recently proved that polynomial filters are linearly stable with respect to the shifted normalised Laplacian matrix $\lap - \mathbf{I}_n$ \cite{polynomialstability}. A simpler proof with a tighter bound (smaller stability constant) was given to show linear stability with respect to the augmented adjacency matrix $\tilde \D^{-1/2}\tilde \A \tilde \D^{-1/2}$ where $\tilde \A = \A + \mathbf{I_n}$ and $\tilde \D = \D + \mathbf{I_n}$ \cite{kenlay2020stability}.
In addition, the following more general result holds.

\begin{prop}
Polynomial filters $g(\lambda)=\sum_{k=0}^K \theta_k \lambda^k$ are linearly stable with respect to any GSO where the spectrum lies in $[-1, 1]$. The stability constant is given by $C=\sum_{k=1}^K k \abs{\theta_k}.$
\end{prop}

Proof: See Supplementary Material. 
Another important class of filters are low-pass filters, which are linearly stable with respect to the normalised Laplacian matrix. 
\begin{prop}
\label{prop:lowpassfilter}
The low-pass filter $g(\lambda)=(1+\alpha \lambda)^{-1}$ is linearly stable with respect to the normalised Laplacian matrix. The constant is given by $C=\alpha$. 
\end{prop}
Proof: See Supplementary Material.
A thorough characterisation of linearly stable filters is beyond the scope of this work. Instead, this section serves to motivate why we are interested in analysing the magnitude of $\norm{\E}_2$: some common types of filters are stable to small perturbation when the perturbation is measured by the error norm $\norm{\E}_2$. Although this is an intuitive choice, it is not immediately clear how $\norm{\E}_2$ is related to the characteristics of the structural properties of the perturbation. This motivates us to provide an upper bound on $\norm{\E}_2$ in Section \ref{sec:interpretable} in terms of interpretable characteristics in the structural domain.

\section{Interpretable bound on filter output change}
\label{sec:interpretable}
In this section we bound $\norm{\E}_2$ by interpretable properties relating to the structural change. Given a perturbation and a node $u$ we denote $\mathcal{A}_u$,  $\mathcal{D}_u$, and $\mathcal{R}_u$ as the set of adjacent nodes for newly added edges, deleted edges, and remaining edges around $u$, respectively.
We denote $\Delta_u^+=\abs{\mathcal{A}_u}$ and $\Delta_u^-=\abs{\mathcal{D}_u} < d_u$ the number of edges added and deleted around $u$, respectively, and $\Delta_u = \Delta_u^+-\Delta_u^-$ as the change of degree. We denote $d_u' = d_u + \Delta_u$ as the degree of node $u$ in $\G_p$. We define $\alpha_u = \max_{v \in \mathcal{N}_u \cup \set{u}} \abs{\Delta_v}/d_v$, where $\mathcal{N}_u$ is the 1-hop neighbourhood of node $u$, as the maximum relative change in degree among a node $u$ and its neighbours. In addition, we define $\delta_u = \min_{v \in \mathcal{N}_u}d_v$ as the smallest degree of the nodes neighbouring node $u$, and $\delta_u'$ as the same quantity in the perturbed graph. We assume that both the graph $\G$ and the perturbed graph $\G_p$ do not contain isolated nodes. 

Our approach to upper bounding $\norm{\E}_2$ relies on the inequality $\norm{\E}_2^2 \leq \norm{\E}_1 \norm{\E}_\infty$ \citep[Section 6.3]{higham2002accuracy}. As $\E$ is Hermitian, $\norm{\E}_1 = \norm{\E}_\infty$ thus simplifying this inequality to become $\norm{\E}_2 \leq \norm{\E}_1.$ There may exist strategies which give tighter bounds, but the benefit of this approach is that $\norm{\E}_1$ leads to an interpretation in the structural domain. 
Thus, we are making use of the following inequality
\begin{equation}
\label{eq:norm2leqmax}
    \norm{\E}_2 \leq \norm{\E}_1 = \max_{u \in \V} \norm{\E_u}_1.
\end{equation}
By considering how the entries of $\lap$ in Eq.~(\ref{eq:Laplacian}) change, we have the following closed-form expression for $\norm{\E_u}_1$:
\begin{align}
\label{eq:Eu}
\norm{\E_u}_1 ={} \sum_{v \in \mathcal{D}_u} \frac{1}{\sqrt{d_u d_v}} + \sum_{v \in \mathcal{A}_u} \frac{1}{\sqrt{d_u'd_v'}} + \sum_{v \in \mathcal{R}_u} \abs{\frac{1}{\sqrt{d_ud_v}} - \frac{1}{\sqrt{d_u'd_v'}}}.
\end{align}
The results of this section bound the three terms in this expression, leading to an overall bound to $\norm{\E_u}_1$ hence to $\norm{\E}_1$ and $\norm{\E}_2$.
We proceed by bounding each of the terms in Eq.~(\ref{eq:Eu}).

\subsection{Bounding the error norm}
\label{subsection:boundEu}
Recall that $\delta_u$ is the smallest degree in the neighbourhood of a node $u$, allowing us to bound the first term in Eq.~(\ref{eq:Eu}) by replacing $d_v$ with $\delta_u \leq d_v$ in the denominator to give:
\begin{equation}
\label{eq:del}
\sum_{v \in \mathcal{D}_u} \frac{1}{\sqrt{d_u d_v}} \leq \sum_{v \in \mathcal{D}_u} \frac{1}{\sqrt{d_u \delta_u}} = \frac{\Delta_u^-}{\sqrt{ d_u \delta_u}}.
\end{equation}
Similarly, we can bound the second term in Eq.~(\ref{eq:Eu}) as:
\begin{equation}
\label{eq:add}
\sum_{v \in \mathcal{A}_u} \frac{1}{\sqrt{d_u'd_v'}} \leq  \sum_{v \in \mathcal{A}_u} \frac{1}{\sqrt{d_u'\delta_u}} = \frac{\Delta^+_u}{\sqrt{ d_u' \delta_u'}}.
\end{equation}
To bound the third term in Eq.~(\ref{eq:Eu}), we first introduce the following lemma.
\begin{lemm}
\label{lem:alpha}
Let $\alpha_u \in [0, 1)$. Then the following holds:
\begin{align}
\sum_{v \in \mathcal{R}_u} \abs{\frac{1}{\sqrt{d_ud_v}} - \frac{1}{\sqrt{d_u'd_v'}}} \leq \sum_{v \in \mathcal{R}_u} \left(\frac{\alpha_u}{1-\alpha_u}\right)\frac{1}{\sqrt{d_u d_v}} \leq  \left(\frac{\alpha_u}{1-\alpha_u}\right)\frac{d_u-\Delta_u^-}{\sqrt{ d_u \delta_u}}.
\end{align}
\end{lemm}
Proof: See Supplementary Material. 
The assumption on $\alpha_u$ can be interpreted as follows.
If $\alpha_u$ = 0 then the degree of $u$ and that of all nodes in the neighbourhood of $u$ are unchanged, so the third term in Eq.~(\ref{eq:Eu}) becomes zero. If $\alpha_u \geq 1$, then for some node $v$ in $\mathcal{N}_u \cup \set{u}$ we have $\abs{\Delta_v}/d_v \geq 1$. Notice that for all nodes we have $\Delta_v > -d_v$, since the degree after perturbation $\delta_v'$ is strictly positive (recall that we do not allow isolated nodes). Therefore, we must instead have $\Delta_v \geq d_v$ which implies $d_v' \geq 2d_v$. In other words, the assumption $\alpha_u < 1$ means that for all nodes $v$ in $\mathcal{N}_u \cup \set{u}$ we have $d_v' < 2d_v$. This limits large amount of change around low degree nodes. It can be noted that if a perturbation does not alter the degree distribution then $\alpha_u=0$ for all nodes and the third term in Eq.~(\ref{eq:Eu}) vanishes. We will consider a particular case of degree preserving perturbation in Section \ref{sec:rewiring}.


By combining the bounds in Eq.~(\ref{eq:del}) and Eq.~(\ref{eq:add}) with Lemma \ref{lem:alpha}, we can further bound Eq.~(\ref{eq:Eu}): 
\begin{align}
    \norm{\E_u}_1 &\leq \frac{\Delta_u^-}{\sqrt{  d_u \delta_u}} +  \frac{\Delta_u^+}{\sqrt{ d_u' \delta_u'}} +  \left(\frac{\alpha_u}{1-\alpha_u}\right)\frac{d_u-\Delta_u^-}{\sqrt{ d_u \delta_u }}. \label{eq:finalbound} 
\end{align}
By further combining this bound with Eq.~(\ref{eq:norm2leqmax}), we arrive at our main result. 
\begin{thm}
\label{thm:main}
Let $\alpha_u \in [0, 1)$ for all nodes $u \in \V$. Then the following holds:
\begin{equation}
    \norm{\E}_2 \leq \max_{u \in \V} \bigg\{ \frac{\Delta_u^-}{\sqrt{  d_u \delta_u}} +  \frac{\Delta_u^+}{\sqrt{ d_u' \delta_u'}} +  \left(\frac{\alpha_u}{1-\alpha_u}\right)\frac{d_u-\Delta_u^-}{\sqrt{ d_u \delta_u }} \bigg\}
    \label{eq:finalbound2} 
\end{equation}
\end{thm}
Proof: See Supplementary Material.
We will explore the looseness of the bound given in Eq.~(\ref{eq:finalbound}) and Eq.~(\ref{eq:finalbound2}) in Section \ref{sec:experiments}. In practice, the bound might be loose but it provides insight into when we expect filters to be stable. We will discuss this insight in Section \ref{subsection:boundintepretation}. In the following subsection, we will consider a special case where we can produce a bound that is tighter in practice.

\subsection{Bounding the error norm under edge rewiring}
\label{sec:rewiring}
Degree preserving rewiring is a type of edge rewiring such that the perturbation does not change the original degree distribution \cite{kenlay2020stability}. Given two edges $u \sim v$ and $u' \sim v'$ such that $u \not \sim v', u \not \sim u', v \not \sim v'$ and $v \not \sim u'$, the double edge rewiring operation deletes the two edges and introduces the edges $u \sim u'$ and $v \sim v'$ (see Fig.~\ref{fig:rewiring} for an illustration). The perturbation consists of two edge deletions and two edge additions and does not change the degree of any nodes involved. This model of perturbation approximately arises in practical applications, where the capacity of a node is fixed and remains at full load such as in communication networks \cite{bienstock1994degree}. 
In this specific scenario, $\alpha_u=0, \delta_u=\delta_u'$ and $d_u=d_u'$. Furthermore, we know that $\Delta_u^- = \Delta_u^+ = r_u$ where we define $r_u$ as the number of rewiring operations involved around a node $u$. Using Theorem \ref{thm:main} we get the following corollary.
\begin{corr} If the perturbation consists of only double edge rewiring operations then:
\begin{equation}
    \norm{\E}_2 \leq \max_{u \in \V} \frac{2r_u}{\sqrt{ d_u \delta_u}}.
\end{equation}
\end{corr}
A similar bound has been recently derived to bound the change in feature representations of certain graph neural network architectures \cite{kenlay2020stability}.

\subsection{Bounding filter output change}
To obtain a full bound on filter output change, we combine together bounds 
that are developed in previous sections. Consider a spectral graph filter $g$ which is linearly stable with respect to the normalised Laplacian matrix. We then have the following bound for the filter output change:
\begin{align}
    \frac{\norm{g_\theta(\lap)\x - g_\theta(\lap_p)\x}_2}{\norm{\x}_2} &\leq  \norm{g_\theta(\lap) - g_\theta(\lap_p)}_2 \nonumber \\
    &\leq  C \norm{\E}_2 
     \leq C \max_{u \in \V} \bigg\{ \frac{\Delta_u^-}{\sqrt{ d_u \delta_u}} +  \frac{\Delta_u^+}{\sqrt{ d_u' \delta_u'}} +  \left(\frac{\alpha_u}{1-\alpha_u}\right)\frac{d_u-\Delta_u^-}{\sqrt{d_u \delta_u}} \bigg\}.
\label{eq:filter_dist_bound}
\end{align}
The first inequality is introduced in Eq.~(\ref{eq:filter_dist}) which relates relative output distance and filter distance. The second inequality comes from our assumption that the graph filter is linearly stable (Eq.~(\ref{eq:linearlystabledefn})) with a stability constant $C$. Finally, we can make use of  Eq.~(\ref{eq:finalbound2}) to establish the third inequality, which provides a structurally interpretable bound on the relative output distance.
We discuss interpretations of this result in the following subsection.

\subsection{Interpretation of the bound}
\label{subsection:boundintepretation}
The bounds given in this section let us reason about sufficient conditions under which a perturbation leads to small change in graph filter output. We can conclude from Eq.~(\ref{eq:norm2leqmax}) that perturbations which cause small changes to $\norm{\E_u}_1$ over all nodes $u$ guarantee small change in terms of $\norm{\E}_2$. When would $\norm{\E_u}_1$ be small for a particular node? If $\alpha_u$ is small ($\alpha_u \approx 0$), then $\alpha_u/(1-\alpha_u) \approx 0$ and $1-\alpha_u/(1-\alpha_u) \approx 1$. Therefore the right hand side of Eq.~(\ref{eq:finalbound}) becomes approximately:
\begin{equation}
\label{eq:approx}
    \norm{\E_u}_1 \approx  \frac{\Delta_u^+}{\sqrt{ d_u' \delta_u'}} + \frac{\Delta_u^-}{\sqrt{ d_u \delta_u}}. 
\end{equation}
This approximation holds, which in turn leads to a small $\norm{\E_u}_1$, if we add and delete edges only between nodes with large degrees. The approximation becomes equality in the case when the degree distribution is preserved such as in Section \ref{sec:rewiring}. When would 
$\norm{\E_u}_1$ be small for all nodes? Intuitively, this requires perturbations to be distributed across the graph, i.e., not concentrated around any one node.
Therefore, spectral graph filters are most robust when we a) add or delete edges between high degree nodes, and b) do not perturb too much around any one node. In the next section, we empirically verify the looseness of each of the bounds and the intuition it provides.

\section{Experiments}
\label{sec:experiments}
We empirically verify the looseness of the bounds derived in the previous section.  
We perform an extensive study of the looseness of these bounds by  considering a variety of experimental conditions in terms of random graph models and perturbation strategies.
Clearly, as the overall bound in Eq.~(\ref{eq:filter_dist_bound}) is obtained from a chain of inequalities, its looseness is affected by the looseness of each individual bound (Eqs.~(\ref{eq:filter_dist}), (\ref{eq:linearlystabledefn}), (\ref{eq:norm2leqmax}), (\ref{eq:finalbound}), (\ref{eq:finalbound2})). 
For completeness, the looseness of the inequalities relating the relative output distance and the filter distance (Eq.~(\ref{eq:filter_dist})) is illustrated in Fig.~\ref{fig:rod-fd}, and that of the inequality relating the filter distance and the constant times the error norm (Eq.~(\ref{eq:linearlystabledefn})) in Fig.~\ref{fig:olexperiments}.

\subsection{Experimental setup}
\label{sec:experimental-setup}
We generate synthetic graphs on $100$ nodes, using different random graph models,  and generate features on the nodes of the graph by taking a random convex combination of the first 10 eigenvectors of the normalised graph Laplacian. The latter results   in relatively smooth signals on the graph. Gaussian noise is then added to generate noisy signals at a signal-to-noise ratio of 0 dB (equal levels of signal and noise). For the sake of simplicity, we focus in this paper on a fixed low-pass filter $f(\lambda)=(1+\lambda)^{-1}$, which has been widely used for signal denoising \citep{ramakrishna2020user} and semi-supervised learning \citep{1326716} tasks. This filter has stability constant $C=1$ and thus satisfies the inequality $\norm{f(\lap)-f(\lap_p)}_2 \leq \norm{\E}_2$, due to Proposition \ref{prop:lowpassfilter}.  We note though that the experiments may be extended to any type of linearly stable graph filter. 
We compare the filtering outcome before and after perturbation to the graph topology in a signal denoising task. To this end, we are interested in bounding the relative output distance between the denoised signal before and after perturbation, for different random graph models and perturbation strategies. The magnitude of the perturbation is set at $\lfloor 10\% \cdot \abs{\mathcal{E}} \rfloor$ edge edits. Each experiment is repeated 100 times using different random seeds. We note that in some experiments the assumption of Lemma \ref{lem:alpha} is not satisfied, i.e., $\alpha_u \ge 1$, and we discuss this further in Supplementary Material.

We use a variety of random graph models including the Erd\H{o}s-R\'enyi model (ER), 
Barab\'asi–Albert model (BA), 
Watts–Strogatz model (WS), 
K-regular graphs (K-Reg), 
K-nearest neighbour graphs (K-NN), and assortative graphs (Assortative) \citep[Chapter 7]{barabasi2013network}. The random graph models include graphs with low variance (K-NN, WS and K-Reg) and high variance (ER, Assortative and BA) in degree distribution. The standard deviation of the degree distribution averaged over each graph type is given in Table \ref{tab:graphs}. Further details of the random graph models are found in Supplementary Material. 

The perturbation strategies under consideration, for a fixed budget, are as follows: 1) randomly selecting edges to delete (Delete); 2) randomly selecting edges to add (Add); 3) using half the budget to randomly add and half to randomly delete (Add/Delete); 4) using degree preserving double edge rewiring as described in Section \ref{sec:rewiring} (Rewire), which is a special case of Add/Delete (we consider a single double edge rewiring to be four edits, i.e., two deletions and two additions); 5) projected gradient descent (PGD), which is used to find adversarial examples by perturbing the graphs similarly to that described in \citet{pgd}; and 6) sequentially deleting or adding edges in a greedy manner to minimise $\norm{\E}_1$ (Robust). Further details of the perturbation strategies are described in Supplementary Material.

\subsection{How tight is the bound $||\E||_2 \leq ||\E||_1$?}
\label{sec:boundontwonorm}
We upper bound $\norm{\E}_2$ using the inequality given in Eq.~(\ref{eq:norm2leqmax}).
In order to quantify the tightness of the bound, we compare in Fig.~\ref{fig:normsa} the values of $\norm{\E}_1$ and $\norm{\E}_2$ for different perturbation strategies, by illustrating their correlation. We note that Robust leads to the smallest values of $\norm{\E}_2$ among all perturbation strategies. This is expected as 
Eq.~(\ref{eq:norm2leqmax}) tells us that small values of $\norm{\E}_1$, which are achieved with the Robust perturbation strategy, guarantee small values of $\norm{\E}_2$. As a matter of fact,  we observe experimentally that the two norms are correlated ($r=0.90$), confirming that in practice if we observe $\norm{\E}_2$ to be small then $\norm{\E}_1$ is likely to be small too. 

We show in Fig.~\ref{fig:normsb} the looseness of the bound given in Eq.~(\ref{eq:norm2leqmax}) among the different perturbation strategies and graph models\footnote{We do not display outliers for figures that appear in main text to make the rest of the data easier to visualise. The same plots with outliers included is included in Supplementary Material.}. Again, we see that the bound is tightest for Robust. It is interesting to observe that Rewire sometimes gives a tight bound for the graphs with low variance in degree distribution.


\subsection{How tight are the bounds on $||\E||_1$ and $||\E||_2$?}
\label{sec:holderexperiments}

We now turn our attention to the bound given in Section \ref{sec:interpretable}. As well as calculating the overall value of $||\E_u||_1$ we can also compute the contributions of the three terms in Eq.~(\ref{eq:Eu}). 
This allows us to evaluate the looseness of each term as well as the overall looseness of Eq.~(\ref{eq:finalbound}).  For each experiment we selected the node $u$ such that $u = \argmax \norm{\E_u}_1$ and calculated the terms and bounds for this node.  The results for each term are shown in Fig.~\ref{fig:bound_all_terms} and that for the overall looseness in Fig.~\ref{fig:bound_overall}.


As one can see from Fig.~\ref{fig:bound_all_terms}, the bounds for particular terms are tight in certain scenarios. For example, the inequality in Eq.~(\ref{eq:del}) becomes equality when $d_v=\delta_u$ for all neighbouring nodes $v$ of $u$, which is the case for 3-Reg graphs. The inequality in  Eq.~(\ref{eq:del}) and Eq.~(\ref{eq:add}) is loosest when $\delta_u \ll d_v$ for many neighbouring nodes $v$. This is likely to occur when the degree distribution has high variance, explaining why the bound for the first and second term are looser for graphs with higher variance in degree distribution. 
The bound on the third term is the loosest in practice. From Fig.~\ref{fig:bound_overall}, the overall bound is tightest for the Rewire strategy, and this is because both the third term and the bound for it are zero in this case. We see that PGD leads to a relatively tight bound as well due to the nature of the strategy based on adversarial examples.

Fig.~\ref{fig:finalbound2} shows the looseness of the bound on $||\E||_2$ in Eq.~(\ref{eq:finalbound2}). The overall pattern is similar to that in Fig.~\ref{fig:bound_overall}, where the bound is tight in some rewiring experiments. In general the bound performs poorly on BA graphs, likely due to the skewed degree distribution.

\begin{figure}
    \centering
    \includegraphics[width=0.5\columnwidth]{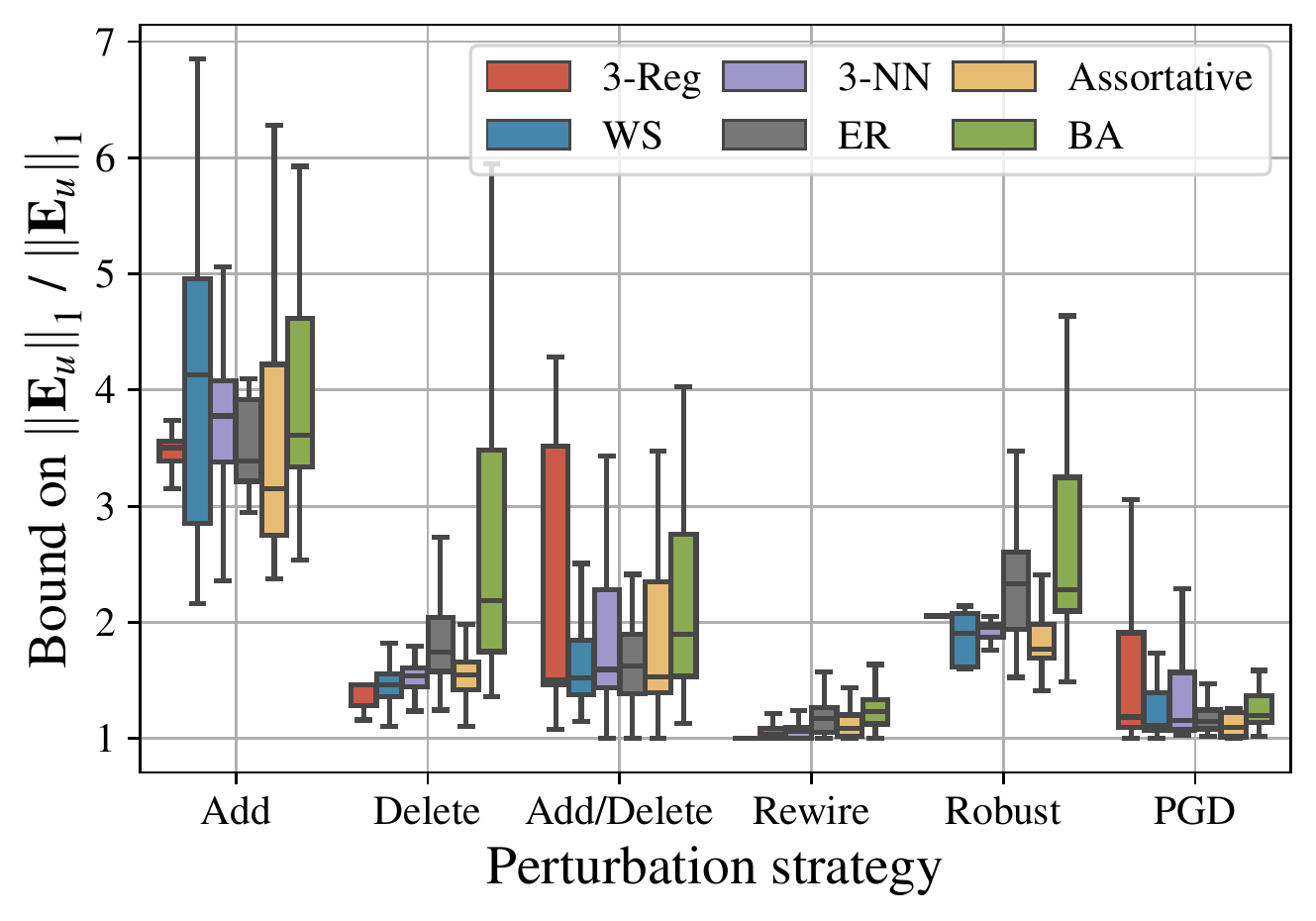}
    \vspace{-0.4cm}
    \caption{Looseness of the bound given in Eq.~(\ref{eq:finalbound}).}
    \label{fig:bound_overall}
\end{figure}

\begin{figure}
    \centering
    \includegraphics[width=0.5\columnwidth]{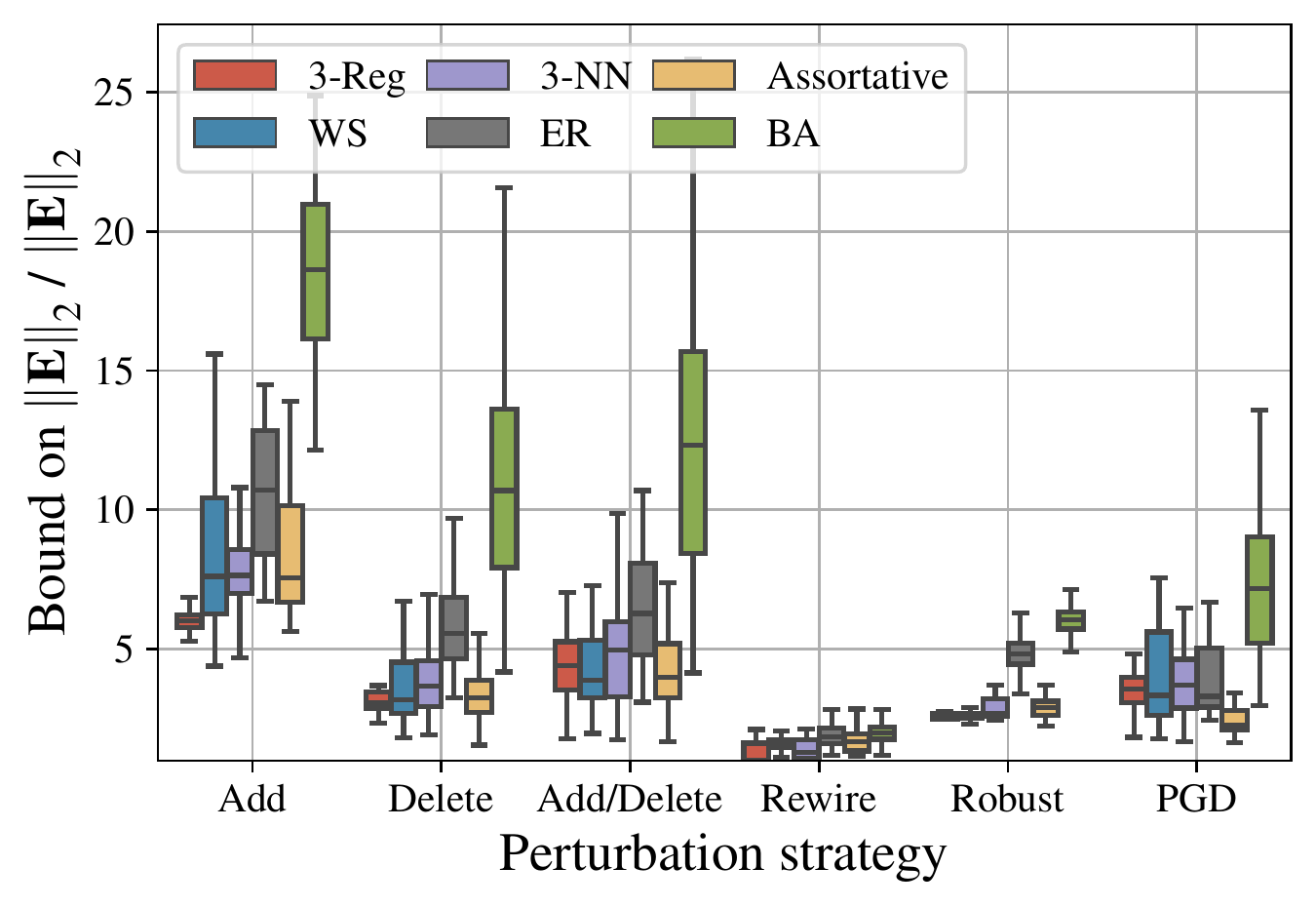}
    \vspace{-0.4cm}
    \caption{Looseness of the bound given in Eq.~(\ref{eq:finalbound2}).}
    \label{fig:finalbound2}
\end{figure}

\subsection{When are filters robust?}
In this section we take a holistic view of the bound in Eq.~(\ref{eq:filter_dist_bound}), considering how the relative output distance is affected by the perturbation strategy, graph model, and the statistics that appear throughout our chain of inequalities. We use insight from our bounds to demonstrate scenarios where a filter is robust. Fig.~\ref{fig:holistic} shows how the different graphs and strategies effect each of the quantities that appear in our bounds. Uniformly, Robust gives overall the smallest values where PGD gives the largest, which is expected due to the nature of these strategies. It is interesting to note that the Robust strategy gives smallest relative output distance for graphs with high degree variance (ER, Assortative and BA) but not for those with low degree variance (3-NN, WS, 3-Reg). When the degree distribution is flat the control and Robust strategies delete and add edges with similar endpoint degrees due to the lack of choice. On the other hand, the PGD strategy gives larger changes to graphs with higher variance in degree distribution, suggesting the existence of low degree nodes or non-uniform degree distributions that are more vulnerable to adversarial attacks.

\begin{figure}
    \centering
    \begin{subfigure}[b]{0.45\columnwidth}
        \centering
        \includegraphics[width=\textwidth]{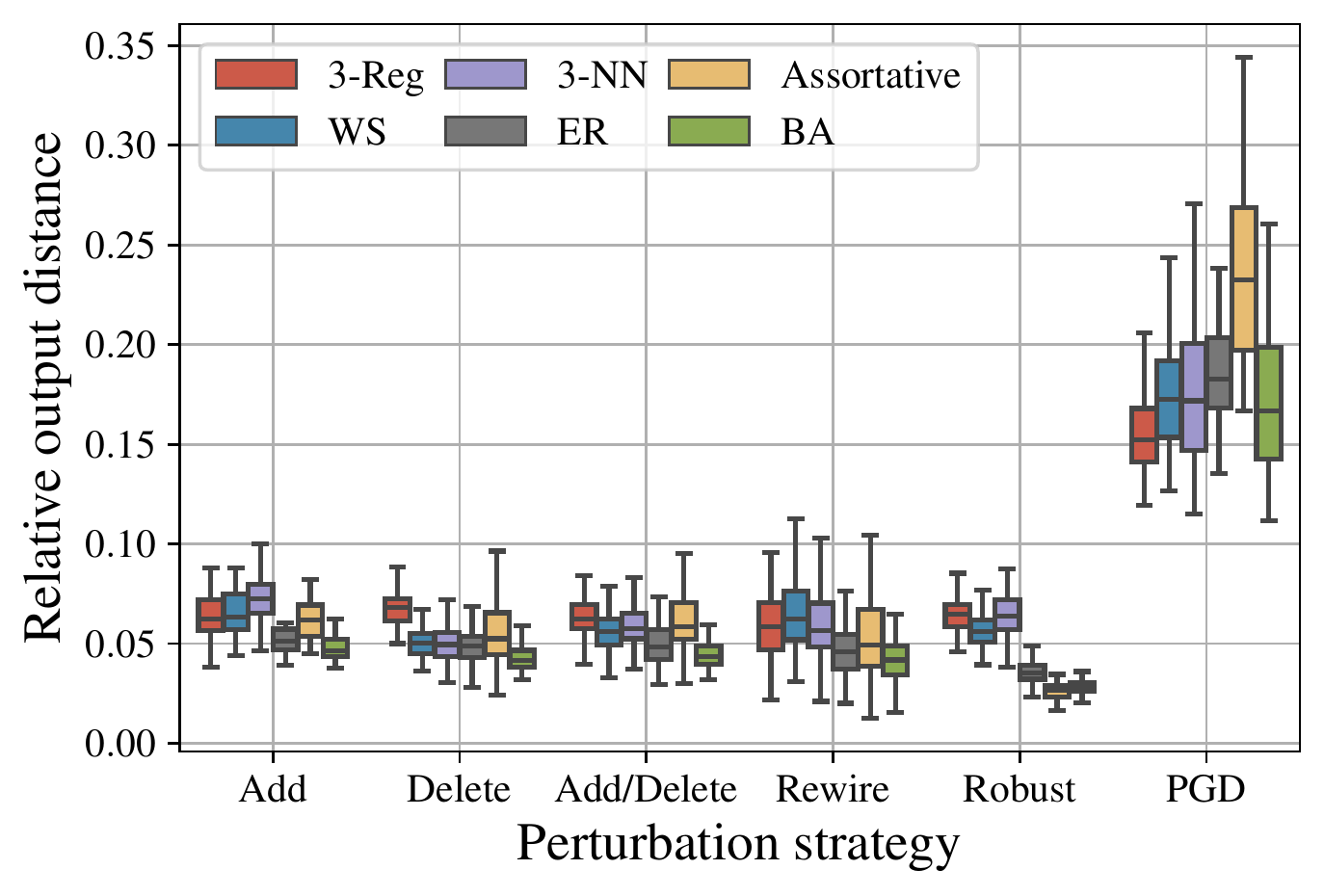}
    \end{subfigure}
    \begin{subfigure}[b]{0.45\columnwidth}
        \centering
        \includegraphics[width=\textwidth]{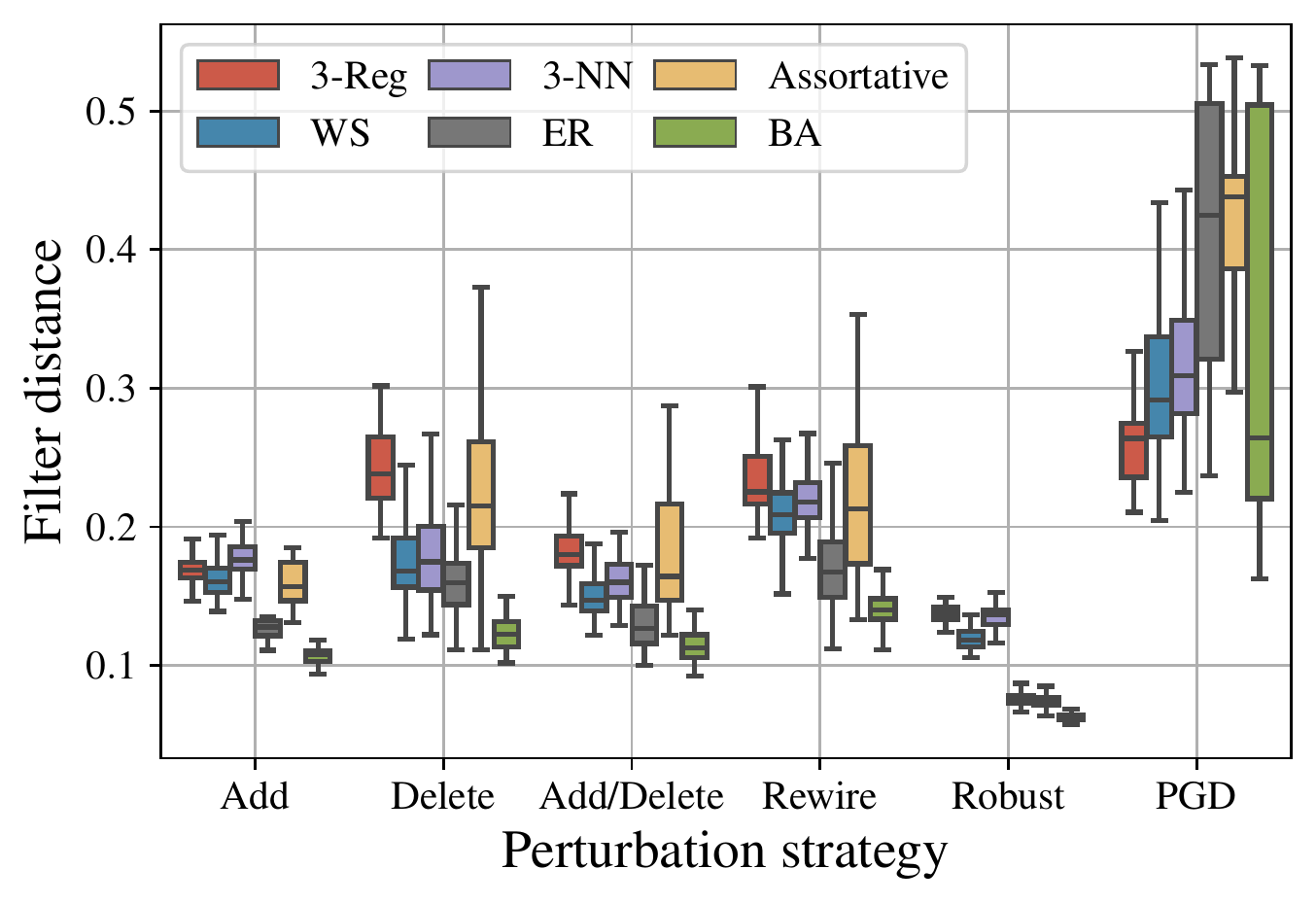}
    \end{subfigure} \\
    
    \begin{subfigure}[b]{0.45\columnwidth}
        \centering
        \includegraphics[width=\textwidth]{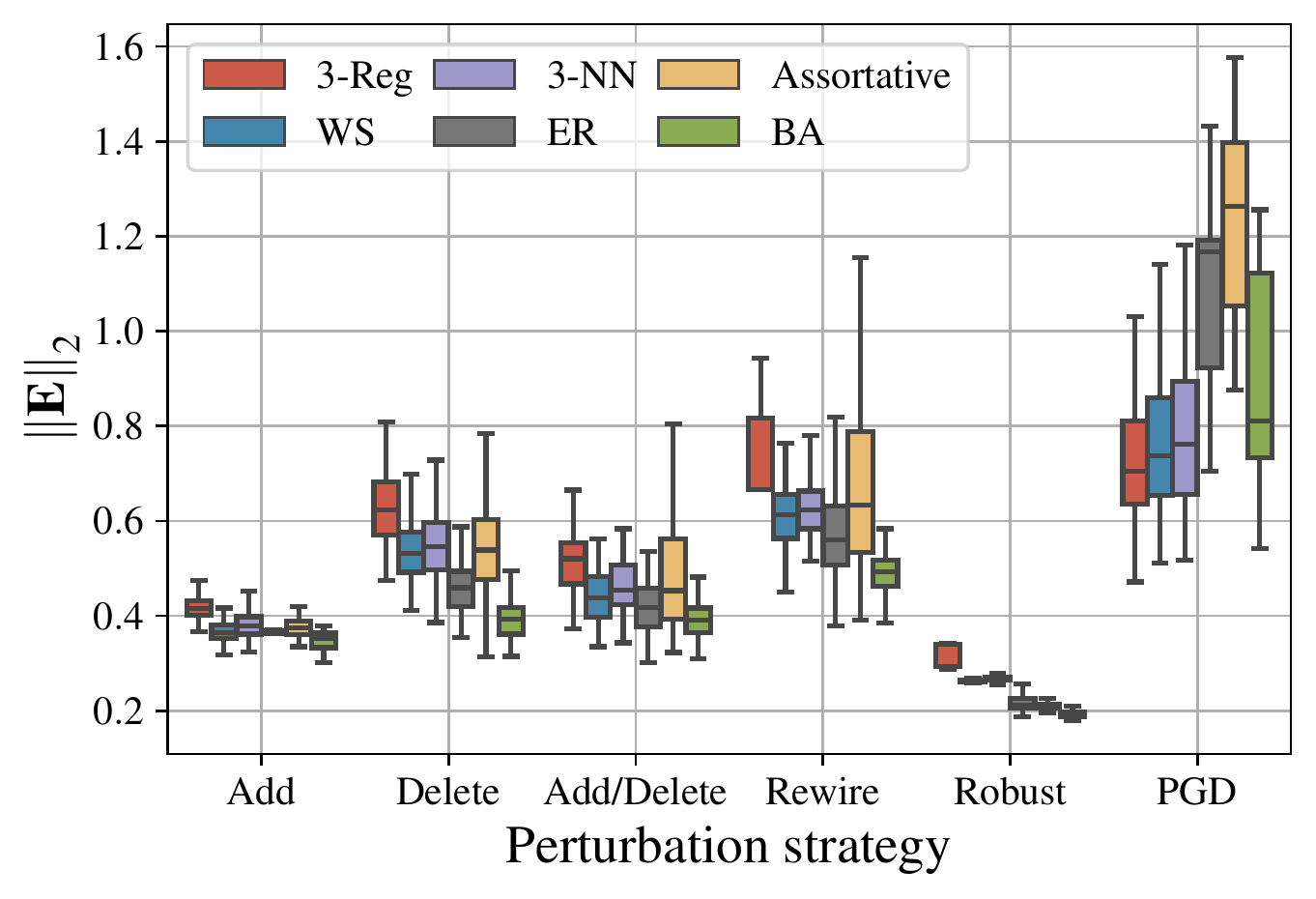}
    \end{subfigure}     
    \begin{subfigure}[b]{0.45\columnwidth}
        \centering
        \includegraphics[width=\textwidth]{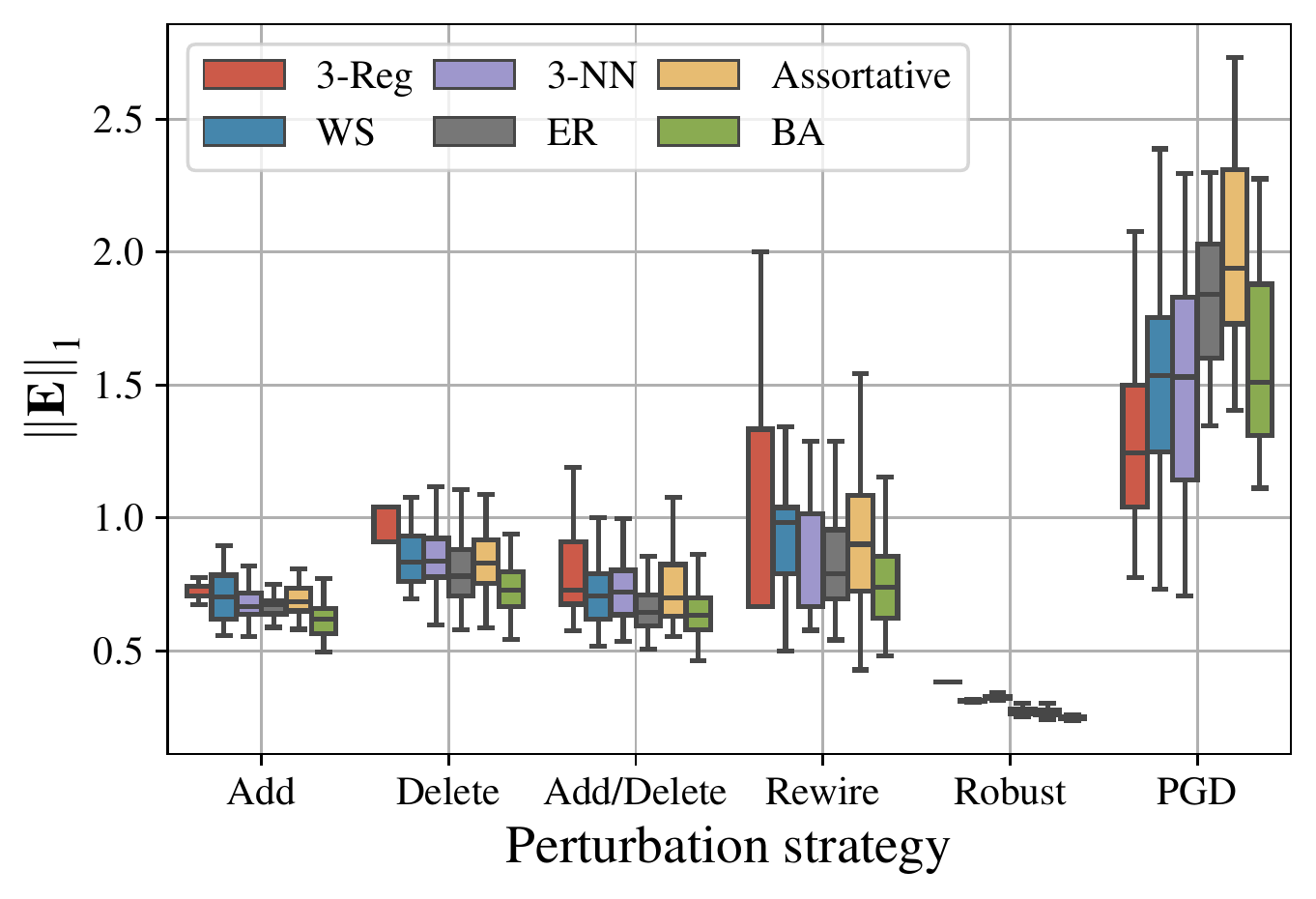}
    \end{subfigure} \\
    
    \begin{subfigure}[b]{0.45\columnwidth}
        \centering
        \includegraphics[width=\textwidth]{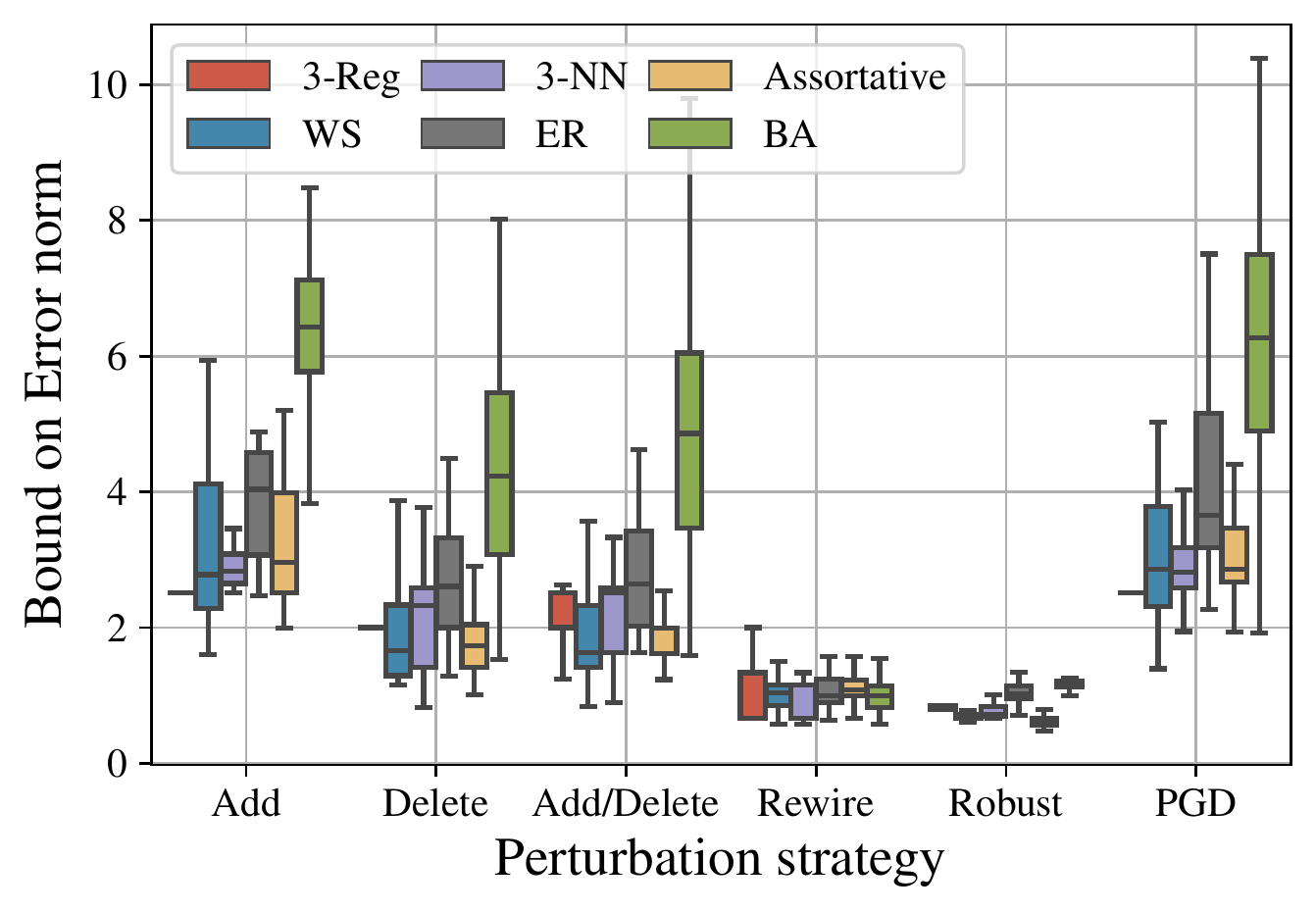}
    \end{subfigure}
    \vspace{-0.4cm}
    \caption{How different statistics vary across experimental setups.}
    \label{fig:holistic}
\end{figure}

In summary, our bounds suggest that filters are robust when we modify edges where the endpoint degrees are high and that the perturbation is distributed across the graph. BA graphs of $n$ nodes have a small diameter that grows asymptotically $\mathcal{O}(\log n / \log \log n)$ \cite{bollobas2004diameter}.
Consequentially, in our experiments on small BA graphs, most edges that are added and deleted are in close proximity. Furthermore, BA graphs have a power-law degree distributions. This type of graph model allows us to control for the distribution of the perturbation and observe instead how the endpoint degrees change across strategies. One can see how PGD tends to target small-degree nodes whereas Robust targets edges connected to the large-degree hubs (Fig.~\ref{fig:BA}). 

\begin{figure}
     \centering
     \includegraphics[width=0.8\columnwidth]{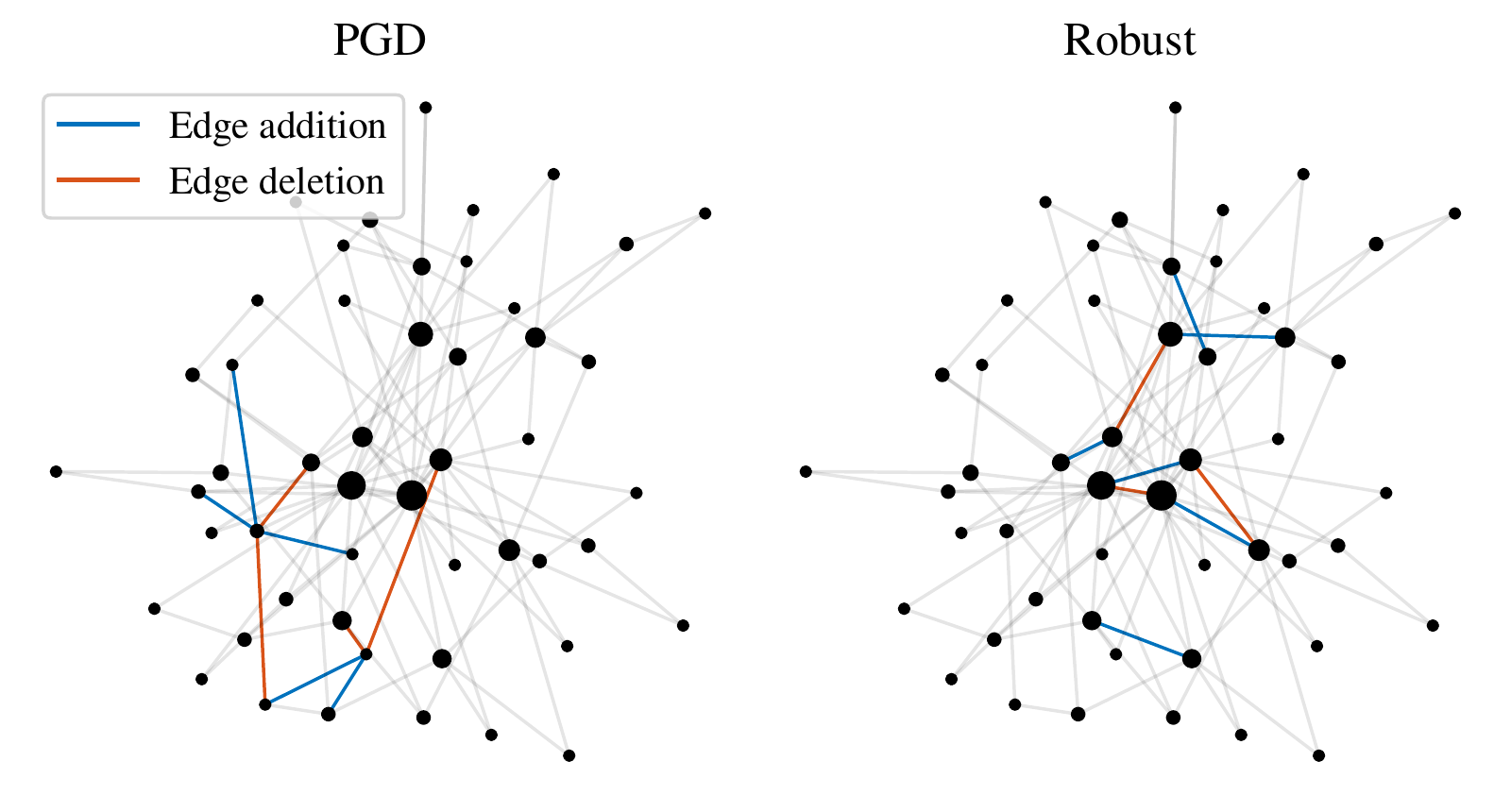}
     \vspace{-0.4cm}
     \caption{Perturbations of BA graphs ($n=50$). The original and both perturbed graphs have a diameter of $5$. The size of the node is proportional to the node degree.}
     \label{fig:BA}
\end{figure}

We finally control for the degree distribution by considering K-regular graphs. 
In Fig.~\ref{fig:kreg} we can see that Robust deletes and adds edges between nodes in a distributed manner, whereas PGD tends to add edges that are adjacent to each other. This verifies the insight from the bound in terms of robustness with respect to spatial distribution of perturbation.

\begin{figure}
     \centering
     \includegraphics[width=0.8\columnwidth]{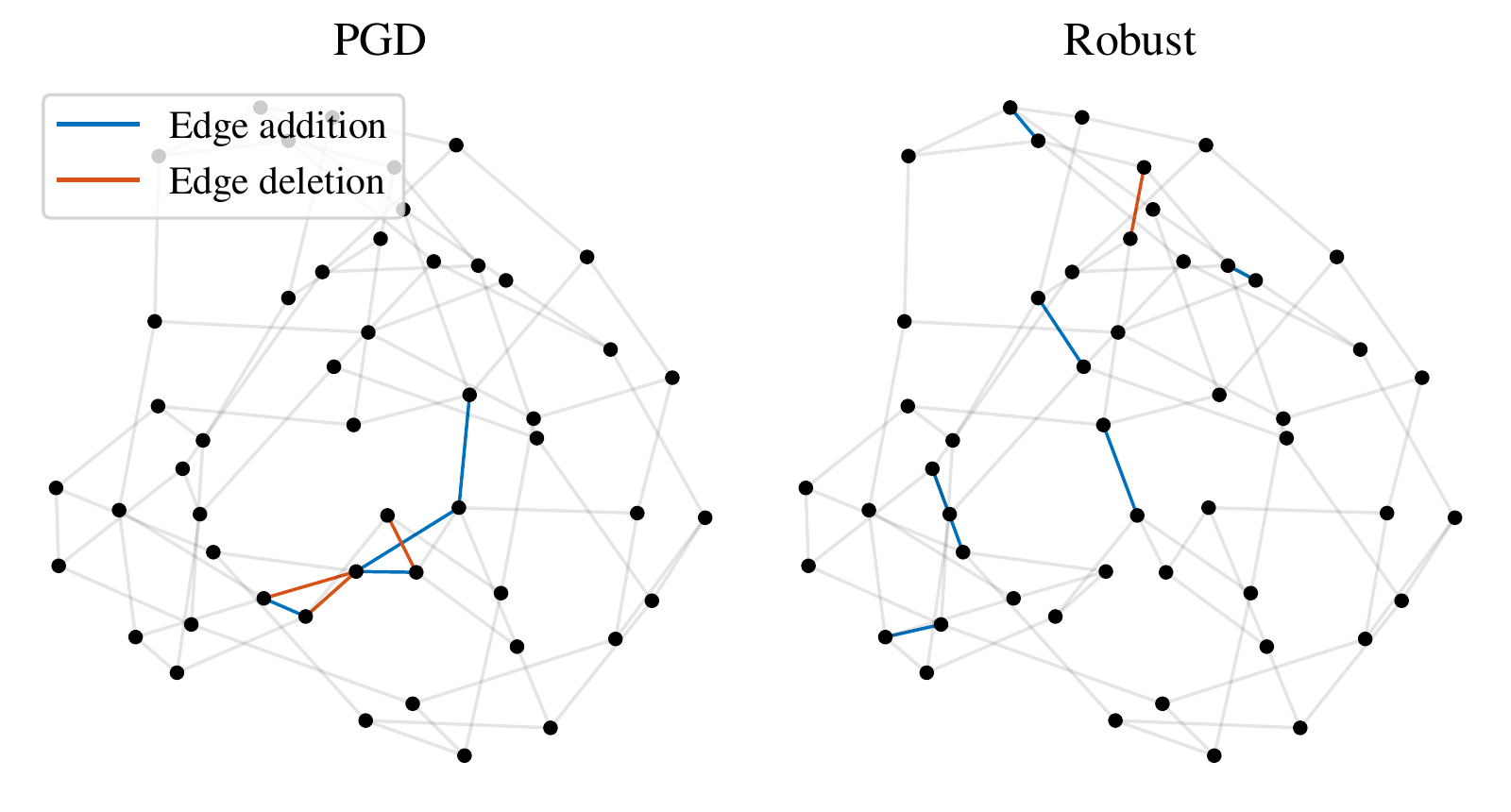}
     \vspace{-0.4cm}
     \caption{Perturbations of 3-regular graphs ($n=50$).}
     \label{fig:kreg}
\end{figure}

\section{Discussion}
\label{sec:conclusion}
In this work, we develop novel interpretable bounds to help elucidate certain types of perturbations against which spectral graph filters are robust. We show that filters are robust when we modify edges where the endpoint degrees are high, and the perturbation is distributed across the graph.  Although these bounds are likely to be loose in practice, they provide qualitative insight which we validate through extensive experiments.

We believe that our work can be used in future research to investigate further the stability of graph-based machine learning algorithms.
Studying additional perturbation models beyond edge deletion/addition, as well as relaxing the assumption on $\alpha_u$, allowing nodes to become isolated, and considering perturbation to node features, may all increase the applicability of the framework to practical scenarios. Considering weighted graphs is another natural extension of the proposed bounds. Our study is limited to a single fixed filter operating on a particular class of graph shift operator. Extension to  a wide range of graph-based machine learning models that might contain multiple spectral graph filters as building blocks is a clear avenue for future research. One such example are graph neural networks, where understanding the stability with respect to perturbations might have positive implications in designing more robust architectures. 

\bibliography{bibliography}
\bibliographystyle{icml2021}

\clearpage
\beginsupplement
\appendix
\section{Proofs}
\label{appendix:proofs}
\setcounter{thm}{0}
\setcounter{prop}{0}
\setcounter{lemm}{0}

\begin{prop}
Polynomial filters $g(\lambda)=\sum_{k=0}^K \theta_k \lambda^k$ are linearly stable with respect to any GSO where the spectrum lies in $[-1, 1]$. The stability constant is given by $C=\sum_{k=1}^K k \abs{\theta_k}.$
\end{prop}
\begin{proof}
The proof technique is the same as in \citet{kenlay2020stability}. Note that $\norm{\lap}_2 \leq 1$ and  $\norm{\lap_p}_2 \leq 1$ so by \citet[Lemma 3]{levie2019transferability} we know that $\norm{\lap^k - \lap_p^k} \leq k \norm{\E}_2$. Using this and the the triangle inequality we have:
\begin{equation*}
    \norm{g(\lap) - g(\lap_p)}_2 = \norm{\sum\limits_{k=1}^K \theta_k (\lap^k - \lap_p^k)}_2 \leq \sum\limits_{k=1}^K \abs{\theta_k} \norm{\lap^k - \lap_p^k}_2 \leq \sum\limits_{k=1}^K k \abs{\theta_k} \norm{\E}_2. \qedhere
\end{equation*}
\end{proof}

\begin{prop}
The low-pass filter $g(\lambda)=(1+\alpha \lambda)^{-1}$ is linearly stable with respect to the normalised Laplacian matrix. The constant is given by $C=\alpha$. 
\end{prop}
\begin{proof}
Let $\mathbf{X}=\mathbf{I}_n + \alpha \lap$ and $\mathbf{Y}=\mathbf{I}_n + \alpha \lap_p$ where $\lap, \lap_p$ are normalised Laplacian matrices. Then we have:
\begin{equation}
    \norm{f(\lap)-f(\lap_p)}_2 = \norm{\mathbf{X}^{-1} - \mathbf{Y}^{-1}}_2 
    = \norm{\mathbf{X}^{-1} (\mathbf{Y}-\mathbf{X}) \mathbf{Y}^{-1}}_2 
    \leq \norm{\mathbf{X}^{-1}}_2 \norm{\mathbf{Y}^{-1}}_2 \norm{\mathbf{X}-\mathbf{Y}}_2.
\end{equation}
Note that $\mathbf{X}$ has eigenvalues in the interval $[1,1+2\alpha]$, so $\mathbf{X}^{-1}$ has eigenvalues in the interval $[(1+2\alpha)^{-1}, 1]$. This holds similarly for $\mathbf{Y}$. The matrices $\mathbf{X}^{-1}$ and $\mathbf{Y}^{-1}$ are symmetric and positive definite so their operator norm is the largest eigenvalue $\norm{\mathbf{X}^{-1}}_2 = \norm{\mathbf{Y}^{-1}}_2 = 1$. Thus,
\begin{equation*}
   \norm{f(\lap)-f(\lap_p)}_2 \leq  \norm{\mathbf{X}^{-1}}_2 \norm{\mathbf{Y}^{-1}}_2 \norm{\mathbf{X}-\mathbf{Y}}_2 
   \leq \norm{\mathbf{X}-\mathbf{Y}}_2 
   = \alpha \norm{\lap-\lap_p}_2. \qedhere
\end{equation*} 
\end{proof}

\begin{lemm}
Let $\alpha_u \in [0, 1)$. Then the following holds:
\begin{align}
\sum_{v \in \mathcal{R}_u} \abs{\frac{1}{\sqrt{d_ud_v}} - \frac{1}{\sqrt{d_u'd_v'}}} &\leq \sum_{v \in \mathcal{R}_u} \left(\frac{\alpha_u}{1-\alpha_u}\right)\frac{1}{\sqrt{d_u d_v}} \label{lemm:firstinequality} \\ 
&\leq  \left(\frac{\alpha_u}{1-\alpha_u}\right)\frac{d_u-\Delta_u^-}{\sqrt{ d_u \delta_u}}. \label{lemm:secondinequality}
\end{align}
\end{lemm}
\begin{proof}
The main part of proving this Lemma is proving the following identity:
\begin{equation}
\label{eq:lemmamaineq}
    \abs{\frac{1}{\sqrt{d_ud_v}} - \frac{1}{\sqrt{d_u'd_v'}}} \leq \left(\frac{\alpha_u}{1-\alpha_u}\right)\frac{1}{\sqrt{d_u d_v}}.
\end{equation}

To prove Eq.~(\ref{eq:lemmamaineq}), we will maximise the left hand size bearing in mind that the constraint $\alpha_u \in [0, 1)$ limits the possible values $d_u'$ and $d_v'$ can take.  To maximise the left hand size we will consider it as a function of $\Delta_u$ and $\Delta_v$ and use partial derivatives to reason where the maxima is.

The left hand side of Eq.~(\ref{eq:lemmamaineq}) can be written as a function of the change in degree of node $u$ and $v$
\begin{equation}
    f(\Delta_u, \Delta_v) = \abs{\frac{1}{\sqrt{d_u d_v}} - \frac{1}{\sqrt{(d_u+\Delta_u)(d_v+\Delta_v)}}}
\end{equation}
on the domain $\Omega = [-\alpha_u d_u, \alpha_u d_u] \times [-\alpha_u d_v, \alpha_u d_v]$. Because the domain is such that $\Omega \subset (0, d_\text{max}]^2$, where $d_\text{max}$ is the largest degree of all nodes in graph $\mathcal{G}$, this function has no poles. The function $f$ has zeros in its domain and the function is non-differentiable at these points but differentiable away from these points. Since the function is non-negative but not identically zero, the local maximas are strictly positive. Therefore we can consider critical points of the function away from the zeros. We consider the partial derivative of $f$ with respect to $\Delta_u$. The chain rule gives us: 
\begin{align}
    \frac{\partial f(\Delta_u, \Delta_v)}{\partial \Delta_u} &= \sign\large(f(\Delta_u, \Delta_v)\large) \ \frac{\partial}{\partial \Delta_u} \left( \frac{1}{\sqrt{d_u d_v}} - \frac{1}{\sqrt{(d_u+\Delta_u)(d_v+\Delta_v)}}\right) \\
    &= \sign\large(f(\Delta_u, \Delta_v)\large) \ \frac{\partial}{\partial \Delta_u} \left(\frac{-1}{\sqrt{(d_u+\Delta_u)(d_v+\Delta_v)}}\right). \label{eq:lemma1proof_1}
\end{align}
Because we only consider the function $f$ away from its zeros, $\sign\large(f(\Delta_u, \Delta_v)\large)$ will be non-zero. Therefore, the partial derivative of $f$ with respect to $\Delta_u$ is zero if and only if the partial derivative on the right hand side of Eq.~(\ref{eq:lemma1proof_1}) is zero. By setting $z=(d_u+\Delta_u)(d_v+\Delta_v)$ we use the chain rule again to obtain:
\begin{equation}
\frac{\partial}{\partial \Delta_u} \frac{-1}{\sqrt{(d_u+\Delta_u)(d_v+\Delta_v)}} = \frac{\partial}{\partial z} \frac{-1}{\sqrt{z}} \frac{\partial z}{\partial \Delta_u} = \frac{1}{2z^{3/2}} (d_v + \Delta_v) = \frac{d_v + \Delta_v}{2\big((d_u+\Delta_u)(d_v+\Delta_v)\big)^{3/2}}.
\label{eq:lemma1proof_2}
\end{equation}

The partial derivative for $f$ with respect to $\Delta_u$ is zero if and only if the right hand side of Eq.~(\ref{eq:lemma1proof_2}) is zero. One can see this partial derivative is zero if and only if $d_v+\Delta_v=0$, but our constraints are such that $d_v + \Delta_v \geq d_v - \alpha_u d_v > 0$. This holds by symmetry for the other partial derivative. Since there are no local maximas in the domain of the function and the function is continuous, the maxima must lie on the boundary of the domain which we write as $\delta \Omega$. 

The boundary $\delta \Omega$ describes a closed rectangle. One can parameterise the sides of this rectangle by either fixing $\Delta_u \in \set{-\alpha d_u, \alpha d_u}$ or by fixing $\Delta_v \in \set{-\alpha d_v, \alpha d_v}$. Without loss of generality consider fixing $\Delta_u = -\alpha d_u$ to give a 1D function in $\Delta_v$ describing a side of the boundary of this rectangle. The derivative of this 1D function is exactly the partial derivative of $f$ with respect to $\Delta_v$. We proved that both partial derivatives are non-zero in $\Omega$, so in particular the derivative of this 1D function is non-zero and thus the maxima must lie on its boundary. The boundary of a side are the two corner end points in the rectangle adjacent to the side. By considering all four sides we can deduce that the maximum of the function must lie on one or more of the four corners. In the following we will prove that the function has a single maximum and show it takes this value at the corner $(-\alpha_u d_u, -\alpha_u d_v)$.

We first show that the function maps two of the corners to the same value: 
\begin{align}
    f(-\alpha_u d_u, \alpha_u d_v) &=  \abs{\frac{1}{\sqrt{d_u d_v}} - \frac{1}{\sqrt{(d_u-\alpha_u d_u)(d_v+\alpha_u d_v)}}} \\
    &=\abs{\frac{1}{\sqrt{d_u d_v}} - \frac{1}{\sqrt{(1-\alpha_u)(1+\alpha_u)}\sqrt{d_u d_v}}} \\
    &=\abs{\frac{1}{\sqrt{d_u d_v}} - \frac{1}{\sqrt{(d_u+\alpha_u d_u)(d_v-\alpha_u d_v)}}} = f(\alpha_u d_u, -\alpha_u d_v).
\end{align}
Next, we show that $f(-\alpha_u d_u, -\alpha_u d_v) \geq f(-\alpha_u d_u, \alpha_u d_v) $. We do this by proving that $f(-\alpha_u d_u, -\alpha_u d_v) - f(-\alpha_u d_u, \alpha_u d_v)$ is non-negative. Writing this out we have that: 
\begin{equation}
f(-\alpha_u d_u, -\alpha_u d_v) - f(-\alpha_u d_u, \alpha_u d_v) =  \abs{\frac{1}{\sqrt{d_ud_v}} - \frac{1}{(1-\alpha_u)\sqrt{d_ud_v}}} - \abs{\frac{1}{\sqrt{d_ud_v}} - \frac{1}{\sqrt{(1-\alpha_u^2)d_ud_v}}}.
\end{equation}
Since $(1-\alpha_u)$ and $\sqrt{1-\alpha_u^2}$ both lie in the interval $(0,1]$, we know that both the values inside the absolute values are non-positive. By negating the values inside and removing the absolute value we get that this is equal to:
\begin{align}
f(-\alpha_u d_u, -\alpha_u d_v) - f(-\alpha_u d_u, \alpha_u d_v)     &= \left(\frac{1}{(1-\alpha_u)\sqrt{d_ud_v}} - \frac{1}{\sqrt{d_ud_v}} \right)  - \left( \frac{1}{\sqrt{(1-\alpha_u^2)d_ud_v}} - \frac{1}{\sqrt{d_ud_v}} \right) \\
    &= \frac{1}{(1-\alpha_u)\sqrt{d_ud_v}} -    \frac{1}{\sqrt{(1-\alpha_u^2)d_ud_v}}.  
\end{align}
This quantity is non-negative if and only if $(1-\alpha_u) \leq \sqrt{(1-\alpha_u^2)}$. By squaring both sides and rearranging we get that this is true if and only if $\alpha_u \in [0, 1]$ which holds true by our assumption.  

Finally, we show that $f(-\alpha_u d_u, -\alpha_u d_v) \geq f(\alpha_u d_u, \alpha_u d_v)$. 
Similar to before we show the following is non-negative: 
\begin{align}
f(-\alpha_u d_u, -\alpha_u d_v) - f(\alpha_u d_u, \alpha_u d_v) &=  \abs{\frac{1}{\sqrt{d_ud_v}} - \frac{1}{(1-\alpha_u)\sqrt{d_ud_v}}} - \abs{\frac{1}{\sqrt{d_ud_v}} - \frac{1}{(1+\alpha_u)\sqrt{d_ud_v}}} \\
&=\left(\frac{1}{(1-\alpha_u)\sqrt{d_ud_v}} - \frac{1}{\sqrt{d_ud_v}} \right) - \left( \frac{1}{\sqrt{d_ud_v}} - \frac{1}{(1+\alpha_u)\sqrt{d_ud_v}} \right) \\
&= \frac{1}{\sqrt{d_ud_v}}\left(\frac{1}{1-\alpha_u} + \frac{1}{1+\alpha_u}-2 \right) \\
&= \frac{1}{\sqrt{d_u d_v}}\left(\frac{-2\alpha_u^2}{(\alpha_u-1)(\alpha_u+1)}\right).
\end{align}
The above is non-negative for $\alpha_u \in [0, 1)$. This proves that $f(-\alpha_u d_u, -\alpha_u d_v)$ is a maxima of the four corners hence a global maxima of the function $f$. We show that the inequality in Eq.~(\ref{eq:lemmamaineq}) holds and is tight by showing equality holds when the left hand side takes its largest value among all possible values of $d_u'$ and $d_v'$ (equivalently $\Delta_u$ and $\Delta_v$):
\begin{align}
\max_{\substack{\Delta_u \in [-\alpha d_u, \alpha d_u] \\ \Delta_v \in [-\alpha d_v, \alpha d_v]}} f(\Delta_u, \Delta_v) = f(-\alpha_u d_u, -\alpha_u d_v) = \frac{1}{(1-\alpha_u)\sqrt{d_ud_v}} - \frac{1}{\sqrt{d_ud_v}} = \left(\frac{\alpha_u}{1-\alpha_u}\right)\frac{1}{\sqrt{d_u d_v}}.
\end{align}

The first inequality (Eq.~(\ref{lemm:firstinequality})) in Lemma \ref{lem:alpha} follows immediately from this. The second inequality (Eq.~(\ref{lemm:secondinequality})) comes from noting that $d_v \geq \delta_u \implies 1/\sqrt{d_ud_v} \leq 1/\sqrt{d_u\delta_u}$ and that $\abs{\mathcal{R}_u} = d_u - \Delta_u^-$ giving:
\begin{equation}
    \sum_{v \in \mathcal{R}_u} \left(\frac{\alpha_u}{1-\alpha_u}\right)\frac{1}{\sqrt{d_u d_v}} \leq \sum_{v \in \mathcal{R}_u} \left(\frac{\alpha_u}{1-\alpha_u}\right)\frac{1}{\sqrt{d_u \delta_u}} = \left(\frac{\alpha_u}{1-\alpha_u}\right)\frac{d_u-\Delta_u^-}{\sqrt{ d_u \delta_u}}.
\end{equation}

\end{proof}

\begin{thm}
Let $\alpha_u \in [0, 1)$ for all nodes $u \in \V$. Then the following holds:
\begin{equation*}
    \norm{\E}_2 \leq \max_{u \in \V} \bigg\{ \frac{\Delta_u^-}{\sqrt{  d_u \delta_u}} +  \frac{\Delta_u^+}{\sqrt{ d_u' \delta_u'}} +  \left(\frac{\alpha_u}{1-\alpha_u}\right)\frac{d_u-\Delta_u^-}{\sqrt{ d_u \delta_u }} \bigg\}.
\end{equation*}
\end{thm}
\begin{proof}
If $\alpha_u \in [0, 1)$ for all nodes then Eq.~(\ref{eq:finalbound}) also holds for all nodes. Substituting Eq.~(\ref{eq:finalbound}) into Eq.~(\ref{eq:norm2leqmax}) gives the desired result.
\end{proof}

\section{Random graph models}
\label{appendix:graphmodels}
In this work we consider graphs without isolated nodes for simplicity. Although the bounds do not require graphs to be connected, we always consider the unperturbed graphs to be connected. To achieve this in practice we use rejection sampling, i.e., sampling from the random graph models until we sample a connected graph. We describe the random graph models we use in our experiments in detail below. Where available we use implementations provided by the NetworkX and PyGSP libraries. Summary statistics describing the degree distribution as well as distance between nodes are given in Table \ref{tab:graphs}.

\begin{table}[h]
\centering
\scalebox{0.95}{
\begin{tabular}{cccccc}
\toprule
             Graph &  Mean degree &  Degree standard deviation &  Average shortest path length &  Diameter & Degree correlation\\
\midrule
        K-Reg &        3.00 &       0.00 &     4.83 &  8.75 & N/A \\
        WS &        4.00 &       0.61 &     5.08 & 10.54 & -0.03 \\
        K-NN &        3.76 &       0.89 &     9.08 & 22.08 & -0.02 \\
        ER &        4.67 &       2.03 &     3.15 &  6.45 & -0.01\\
        Assortative &        4.64 &       2.04 &     4.29 & 11.44 & 0.80 \\
        BA &        5.82 &       4.71 &     2.59 &  4.41  & -0.15 \\
\bottomrule
\end{tabular}}
\caption{Summary statistics averaged across the generated graphs. The first two columns give the average and standard deviation of the degree sequence. The average shortest path length and diameter (largest shortest path length) measure node connectivity. The degree correlation measures assortativity.}
\label{tab:graphs}
\end{table}

\paragraph{Erdős-Rényi} Erdős-Rényi graphs are constructed by independently including each possible edge between any pair of nodes with probability $p$ \citep{gilbert1959random}. For sufficiently large graphs the graph will be connected with high probability if $p > \ln n /n$ and disconnected if $p < \ln n /n$. We thus chose $p= \ln n /n$. The degree distribution of Erdős-Rényi graphs is approximately a Poisson distribution. 

\paragraph{Barabási–Albert} Barabási–Albert graphs are randomly generated scale-free (with power-law degree distributions) graphs which are constructed using a preferential attachment mechanism \citep{albert2002statistical}. The graphs are constructed using parameters $n$, the number of nodes, and $m$, the number of edges that are preferentially attached. An initial graph is given by a star graph on $m+1$ nodes. Then, until the graph has $n$ nodes the following step is repeated. A new node is introduced and connected to $m$ existing nodes with probabilities proportional to their degrees. Barabási–Albert graphs are connected by construction. 

\paragraph{Watts-Strogatz} Watts-Strogatz graphs with appropriate parameters exhibit small-world properties such as small average path lengths and high clustering \citep{watts1998collective}. The graph begins as a ring lattice which is obtained by taking a cycle of nodes and connecting each node to its $K$ nearest neighbours. Then, each edge is rewired independently with probability $p$, which means deleting edge $i \sim j$ and adding edge $i \sim j'$ such that $j' \not \in \set{i, j}$ is selected uniformly at random. We connected each node to its $K=4$ neighbours for the initial configuration and rewire each edge with probability $p=0.1$. 

\paragraph{K-regular}
A K-regular graph is one such that the degree of every node is $K$. We use the algorithm described in \citet{steger1999generating} and implemented in NetworkX. For a fixed $K$, the algorithm samples graphs with $n$ nodes all of degree $K$ almost uniformly at random in the sense that the distribution approaches uniform in the limit $n \to \infty$. 

\paragraph{K-nearest neighbour}
A K-nearest neighbour graph is constructed by sampling uniformly at random $n$ points in the unit square $[0,1]^2$ to form the nodes. Then for each node an edge is connected to its $K$ nearest neighbours in terms of Euclidean distance. 

\paragraph{Assortative}
An assortative graph is one such that there is a high positive correlation between the degree of end points along edges. To generate assortative graphs we use a variant of the Xulvi-Brunet \& Sokolov (XBS) Algorithm \cite{xulvi2004reshuffling} applied to Erdős-Rényi graphs. The XBS algorithm iteratively rewires edges in the graph to increase assortativity whilst keeping the degree sequence fixed. At each step two edges, corresponding to four nodes, are chosen uniformly at random. With probability $p$, these edges are rewired in a way that increases the assortativity. Otherwise, the edges are rewired randomly. In our experiments we set this probability to $p=1$, meaning the assortativity increases in each step. In the original algorithm a graph can become disconnected; we discard an iteration if it disconnects the graph. Instead of running the algorithm for a fixed number of steps we run the algorithm until the degree correlation is at least $0.8$. We describe the algorithm in Algorithm \ref{alg:XBS}. 

\begin{algorithm}[h]
   \caption{XBS Algorithm}
   \label{alg:XBS}
\begin{algorithmic}[1]
   \STATE {\bfseries Input:} An initial graph $\mathcal{G}$, assortative threshold $a$
   \REPEAT
   \STATE Initialize $\mathcal{G}_{temp} \leftarrow \mathcal{G}$.
   \STATE Sample edges $u \sim v$ and $u' \sim v'$ from $\mathcal{G}_{temp}$ such that $u,v,u',v'$ are unique
   \STATE Delete edges $u \sim v$ and $u' \sim v'$ from $\mathcal{G}_{temp}$
   \STATE Connect the two nodes in $\mathcal{G}_{temp}$ from $\set{u,v,u',v'}$ with the highest degree\;
   \STATE Connect the two nodes in $\mathcal{G}_{temp}$ from $\set{u,v,u',v'}$ with the lowest degree\;
   \IF{$\mathcal{G}_{temp}$ is connected}
       \STATE $\mathcal{G} \leftarrow \mathcal{G}_{temp}$
   \ENDIF
   \UNTIL{Assortativity of $\mathcal{G}$ more than or equal to $a$}
   \STATE {\bfseries Output:} Perturbed graph $\mathcal{G}$
\end{algorithmic}
\end{algorithm}

\section{Perturbation strategies}
\label{appendix:perturbationstrategies}
We make use of four random strategies (Add, Delete, Add/Delete and Rewire), one strategy which gives perturbations according to an optimisation based search strategy (PGD), and a strategy which is based on the theory described in Section \ref{sec:interpretable} (Robust). All strategies delete or add a fixed number of edges. The rewiring operation is depicted graphically in Fig.~\ref{fig:rewiring}. We describe Robust and PGD in more detail.

\begin{figure}
    \centering
    \begin{subfigure}{0.4\columnwidth}
        \centering
        \input{figures/rewire-before}
        \caption{Before rewiring}
    \end{subfigure}
    \begin{subfigure}{0.4\columnwidth}
        \centering
        \input{figures/rewire-after}
        \caption{After rewiring}
    \end{subfigure}
    \caption{In the rewiring operation the red edges are deleted and the blue edges are added. The degree of each node remains the same.}
    \label{fig:rewiring}
\end{figure}
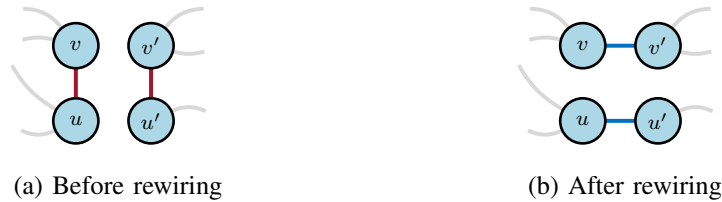

\paragraph{Robust} The robust strategy works by iteratively building the perturbed graph. Consider two unique nodes $u$ and $v$, then we say an edge is flipped between them if an edge $u \sim v$ exists and is deleted, or if the edge does not exist and is added. A single iteration consists of finding nodes $u$ and $v$ that have not been flipped in previous iterations such that $\norm{\E}_1$ is minimal. The steps of the algorithm are described in Algorithm \ref{alg:robust}.

\begin{algorithm}[h]
   \caption{Robust Algorithm}
   \label{alg:robust}
\begin{algorithmic}[1]
   \STATE {\bfseries Input:} An initial graph $\mathcal{G}$, a budget $B$
   \STATE Initialize set of flipped edges $\mathcal{F} \leftarrow \emptyset$.
   \WHILE{$\abs{\mathcal{F}} < B$}
        \STATE Initialize candidate edge $e$
        \STATE Initialize candidate edge norm $f \leftarrow \infty$
        \FOR{Potential edge $e' \in \set{1,\ldots,n}^2 \setminus \mathcal{F}$}
            \STATE Initialize $\mathcal{G}_{temp} \leftarrow \mathcal{G}$
            \STATE Flip edge $e'$ in $\mathcal{G}_{temp}$
            \STATE Calculate normalised Laplacian matrix $\lap_{temp}$ of $\mathcal{G}_{temp}$
            \IF{$\norm{\lap_{temp}}_1 < f'$}
                \STATE $e \leftarrow e'$
                \STATE $f \leftarrow \norm{\lap_{temp}}_1$
            \ENDIF
        \ENDFOR
        \STATE Flip edge $e$ in $\mathcal{G}$
        \STATE $F \leftarrow F \cup \set{e}$
   \ENDWHILE
   \STATE {\bfseries Output:} Perturbed graph $\mathcal{G}$
\end{algorithmic}
\end{algorithm}

\paragraph{PGD} Projected gradient descent is a variant of gradient descent where after each gradient update a projection operation is applied. This strategy aims to find adversarial examples and follows closely the strategy described in \citet{pgd}, with three modifications. The first modification is that during the sampling procedure \citep[Algorithm 1 (step 5)]{pgd}, we not only discard perturbed graphs that are over budget, but also those containing isolated nodes.. The second modification is that we perform gradient ascent on the relative error between the original signal (before noise is added) and the denoised signal (estimate). This is different from the negative cross-entropy or Carlili-Wagner loss functions used in \citet{pgd} which are better suited for classification problems. Finally, our filter involves the inverse of a matrix which can become singular making it not possible to propagate gradients. We describe how we calculate the gradient used in \citet[Algorithm 2 (step 3)]{pgd} in Algorithm \ref{alg:pgd}. The variables $\mathbf{A'}, \mathbf{s}$ and $\mathbf{S}$ are described further in \citet{pgd}. We include additional steps 3 and 4 to improve the stability of the filtering operation (step 6). Step 4 projects negative values to $0$ and values above $1$ to $1$. Step 6 makes use of the lower–upper (LU) decomposition for solving linear systems. We use a learning rate of $\eta_t = 200/\sqrt{t}$, and implement $T=200$ iterations and $K=250$ random trials. 

\begin{algorithm}[h]
\caption{Numerically stable gradient calculation}
\label{alg:pgd}
\begin{algorithmic}[1]
    \STATE {\bfseries Input:} Adjacency matrix $\mathbf{A}$, probability vector $\mathbf{s}$, target signal $\mathbf{y}$, noisy signal $\x$.
    \STATE $\mathbf{A'} = \mathbf{A} + (\mathbf{\bar A} - \mathbf{A}) \circ \mathbf{S}$ \citep[Eq.~4]{pgd} where $\mathbf{S}$ is the matrix form of $\mathbf{s}$
    \STATE $\mathbf{A'} = \mathbf{A'} + \diag (\mathcal{N}(\mathbf{0}_n, 10^{-3} \times \mathbf{I}_n))$
    \STATE $\mathbf{A'} = \clamp (\mathbf{A'})$
    \STATE Calculate Laplacian $\lap'$ of $\mathbf{A'}$ 
    \STATE Calculate denoised signal $\mathbf{\hat y} = (\mathbf{I}_n+\lap')^{-1}\x$
    \STATE Calculate relative error loss $\mathcal{L} = \norm{\mathbf{y} -  \mathbf{\hat y}} / \norm{\mathbf{y}}$
    \STATE Calculate gradient $\nabla_s \mathcal{L}$
    \STATE {\bfseries Output: $\nabla_s \mathcal{L}$}
\end{algorithmic}
\end{algorithm}



\section{Additional results}
\label{appendix:results}

\subsection{How close are the relative output distance to the filter distance?}
\label{sec:rod-fd}
In Section \ref{sec:linearlystable} we have established a relationship between the relative output distance in Eq.~(\ref{eq:rod}) and the filter distance in Eq.~(\ref{eq:filter_dist}). The looseness of the bound given in Eq.~(\ref{eq:filter_dist}) in shown in Fig.~\ref{fig:rod-fd}. As we can see the PGD strategy gives rise to a tighter inequality compares to other strategies.

\begin{figure}
\centering
    \centering
    \includegraphics[width=0.5\textwidth]{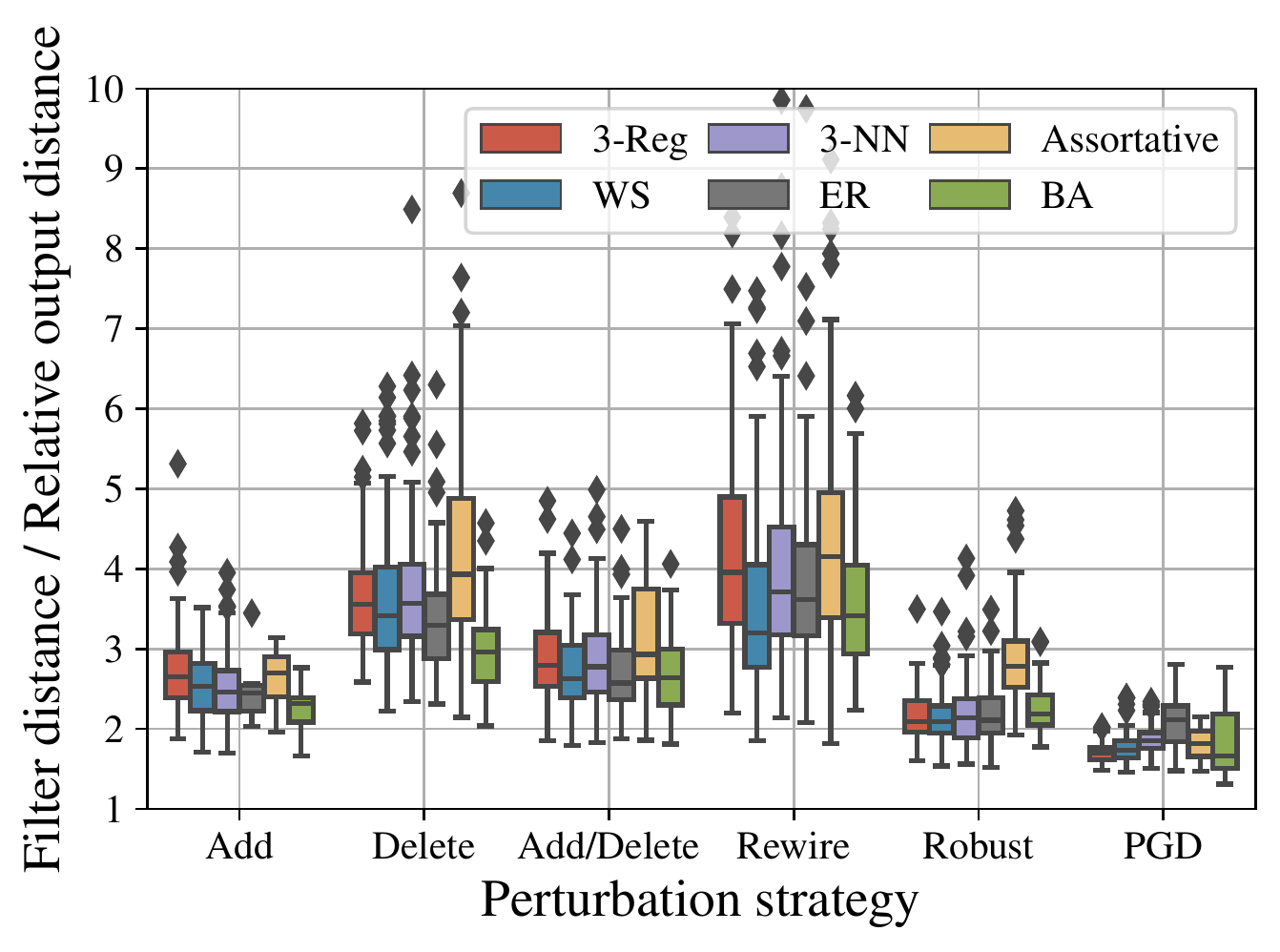}
\caption{Looseness of the bound relating the relative output distance to the filter distance.}  
\label{fig:rod-fd}
\end{figure}

\subsection{How tight is the linear stability bound?}
Linearly stable filters have the property that there exists an upper bound on the filter distance which is linearly proportional to the error norm. In the previous works of \citet{levie2019transferability} and \citet{kenlay2020stability} it has also been shown experimentally that the relationship is linear in practice. The low-pass filter we use in experiments has a stability constant of one meaning $\norm{f(\lap) - f(\lap_p)}_2 \leq \norm{\E}_2$. Our experiments (Fig.~\ref{fig:olexperiments}) show the looseness of this bound to be at least $1.5$ in all our experiments.

\begin{figure}[t]
\centering
\includegraphics[width=0.5\textwidth]{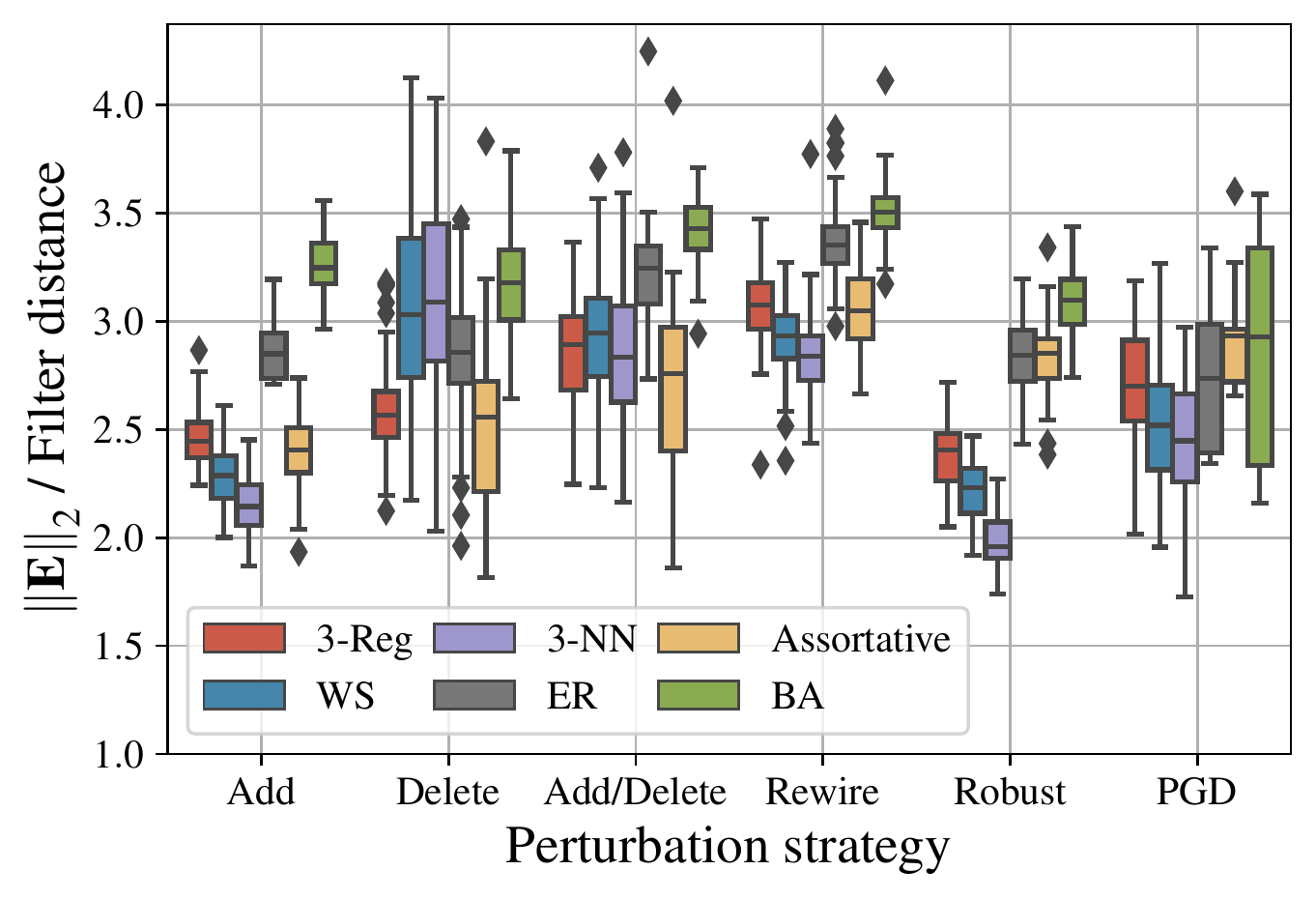}
\caption{Looseness of the linear stability bound.}    
\label{fig:olexperiments}
\end{figure}

\subsection{Validity of experiments}
As mentioned in Section \ref{sec:experiments}, some experiments give $\alpha_u > 1$ for some node $u$ in the graph meaning that strictly speaking the bound of Theorem \ref{thm:main} cannot be applied. Therefore, we discard the results of these experiments from our experimental analysis in the paper. Fig.~\ref{fig:is_valid} shows the proportion out of $100$ experiments we run which are valid (in the sense that the assumption on $\alpha_u$ is satisfied) for all combinations of random graph model and perturbation strategy. As noted in Section \ref{sec:interpretable}, for the assumption on $\alpha_u$ to be violated the degree of at least one node must at least double after perturbation. Since Delete only decreases the degree of nodes, and Rewire preserves the degree of all nodes, experiments with these perturbation strategies are always valid as expected. The graph models with high-variance degree distributions (ER, BA, Assortative) tend to have the assumption violated more often. This is possibly due to a large number of leaf nodes, to which if a single edge is added the assumption on $\alpha_u$ would be violated.

\begin{figure*}
    \centering
    \includegraphics[width=0.5\textwidth]{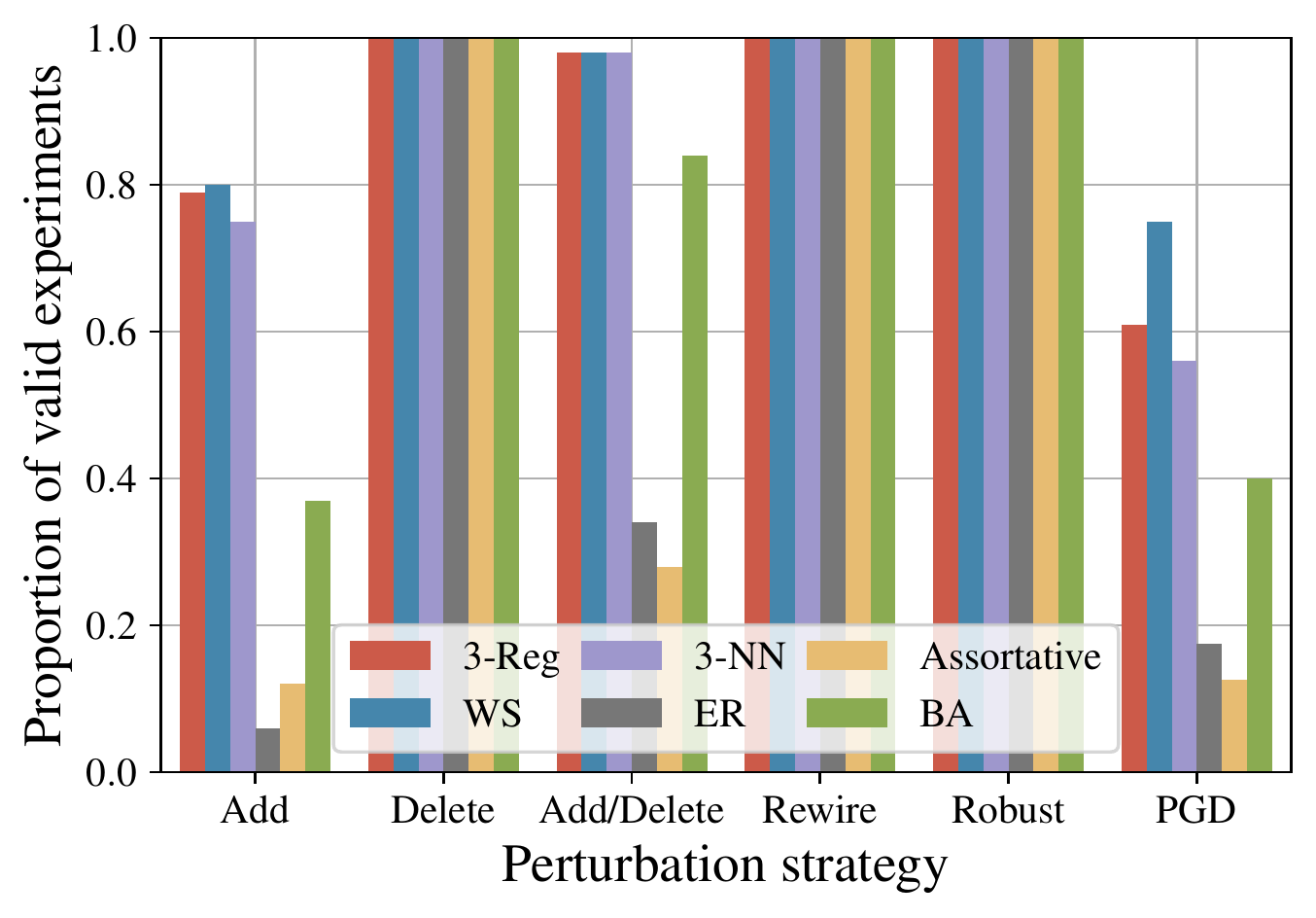}
    \caption{Proportion of experiments which are valid ($\alpha_u \in [0, 1)$ for all nodes $u \in \V$) across all graph types and perturbation strategies.}
    \label{fig:is_valid}
\end{figure*}

\subsection{How tight is the bound $||\E||_2 \leq ||\E||_1$?}

In Section \ref{sec:boundontwonorm} we consider the tightness of the inequality $||\E||_2 \leq ||\E||_1$ for various random graph models and perturbation strategies. For each strategy we plot the absolute value of $||\E||_2$ and $||\E||_1$ for all random graph models in Fig.~\ref{fig:normsa}. We plot just 5 repeats for clarity. In Fig.~\ref{fig:normsb} we plot the looseness of the bound.  

\begin{figure*}
    \centering
    \includegraphics[width=0.4\textwidth]{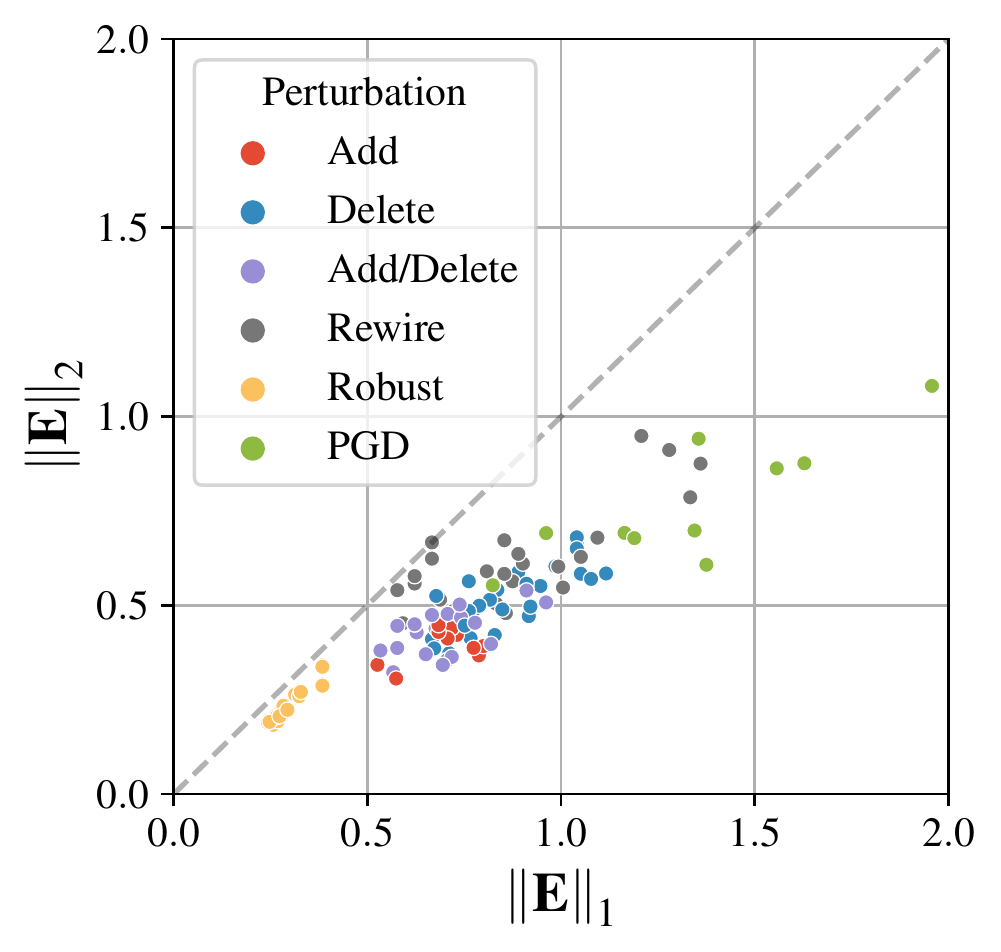}
    \caption{Comparison of $\norm{\E}_1$ and $\norm{\E}_2$. The dashed line represents $\norm{\E}_2=\norm{\E}_1$.}
    \label{fig:normsa}
\end{figure*}

\begin{figure*}
    \centering
    \includegraphics[width=0.5\textwidth]{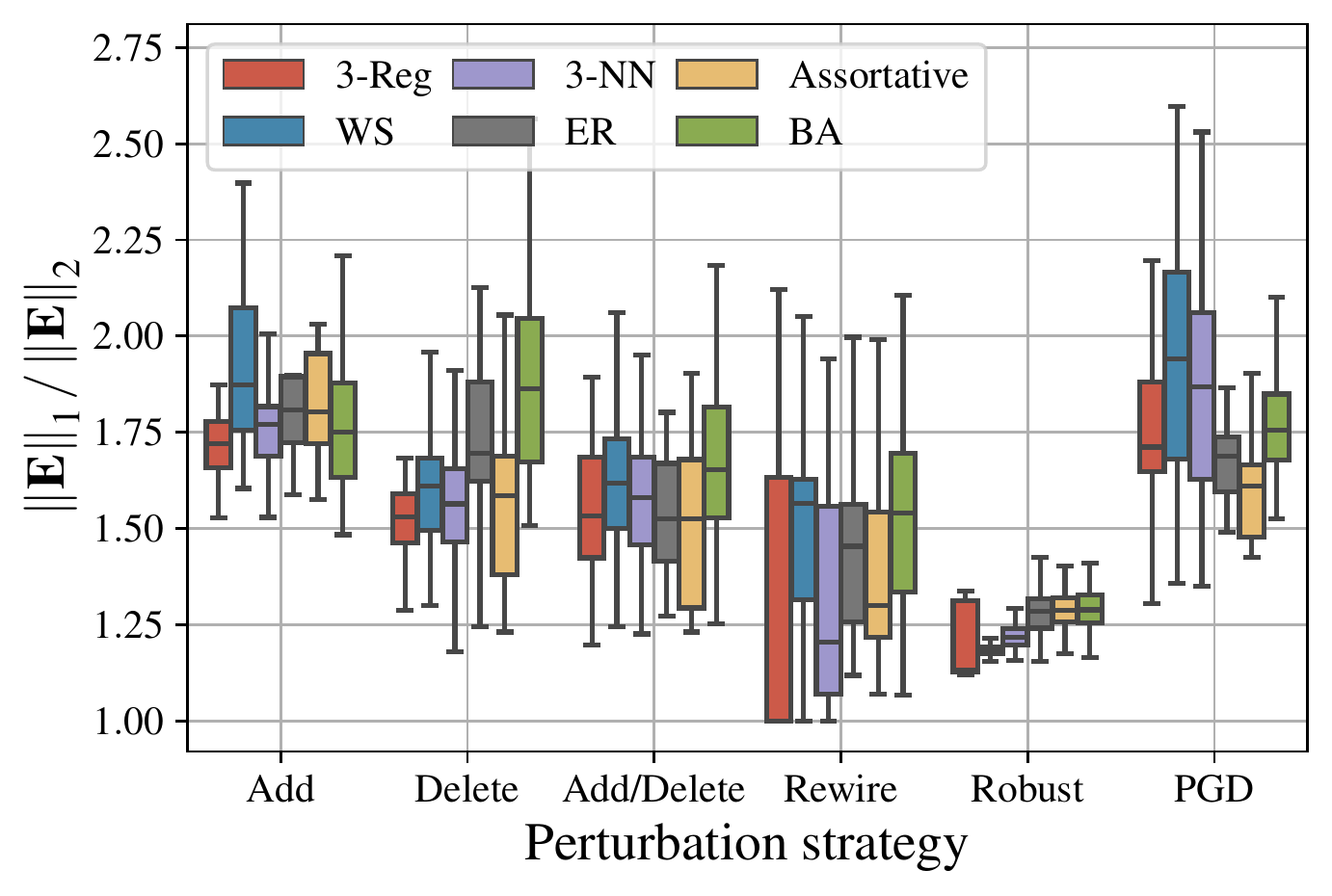}
    \vspace{-0.4cm}
    \caption{Looseness of the inequality $\norm{\E}_1 \leq \norm{\E}_2$ under different perturbation strategies and random graph models.}
    \label{fig:normsb}
\end{figure*}

\subsection{How tight are the bounds on $||\E||_1$ and $||\E||_2$?}
In Section \ref{sec:holderexperiments} we consider the looseness of Eq.~\ref{eq:finalbound}. We can consider the looseness for each component seperately which is shown in Fig.~\ref{fig:bound_all_terms}. It can be seen that the bound on the third term is the loosest. 

\begin{figure*}
    \centering
    \begin{subfigure}[b]{0.3\textwidth}
        \centering
        \includegraphics[width=\textwidth]{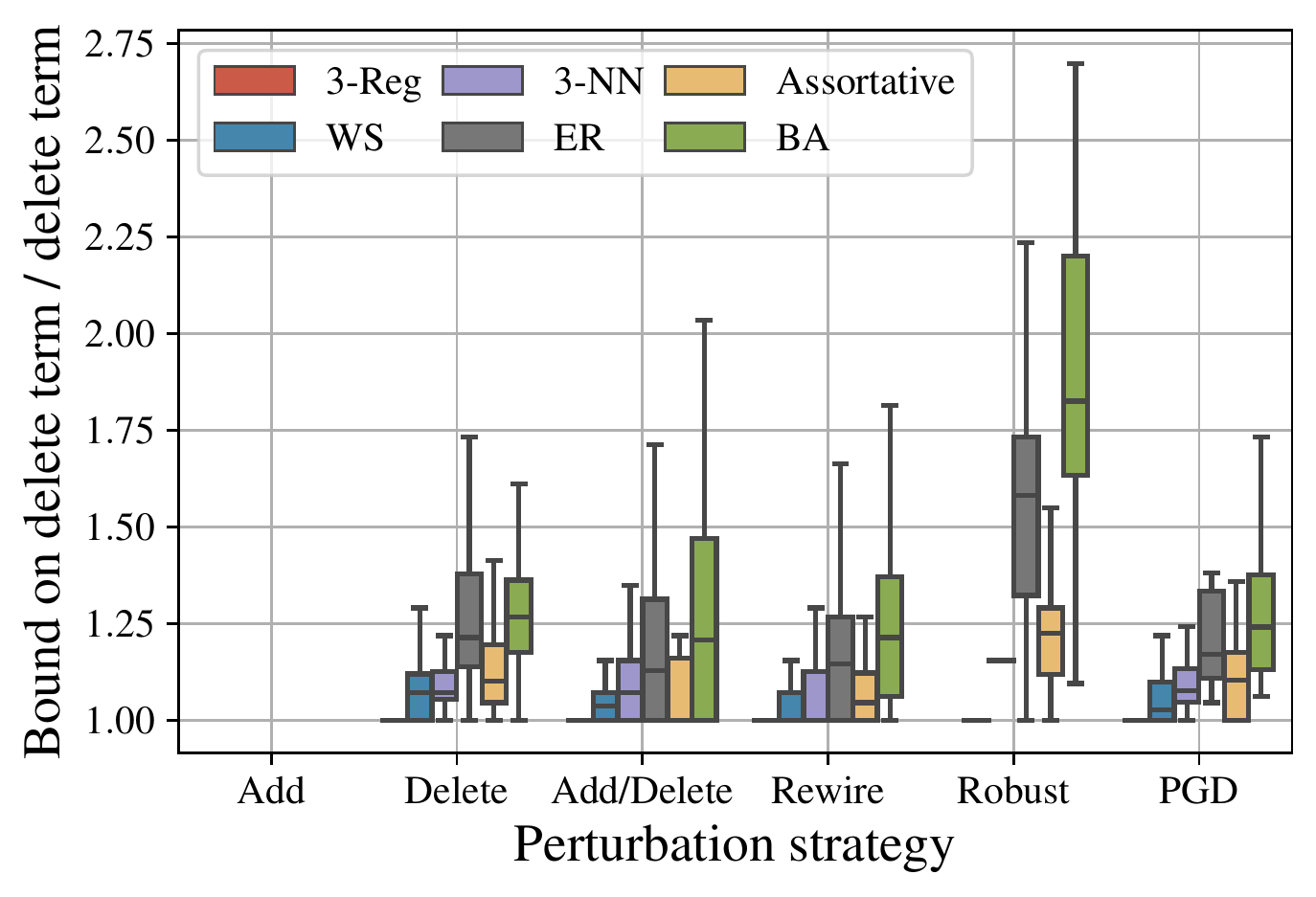}
    \end{subfigure}
        \begin{subfigure}[b]{0.3\textwidth}
        \centering
        \includegraphics[width=\textwidth]{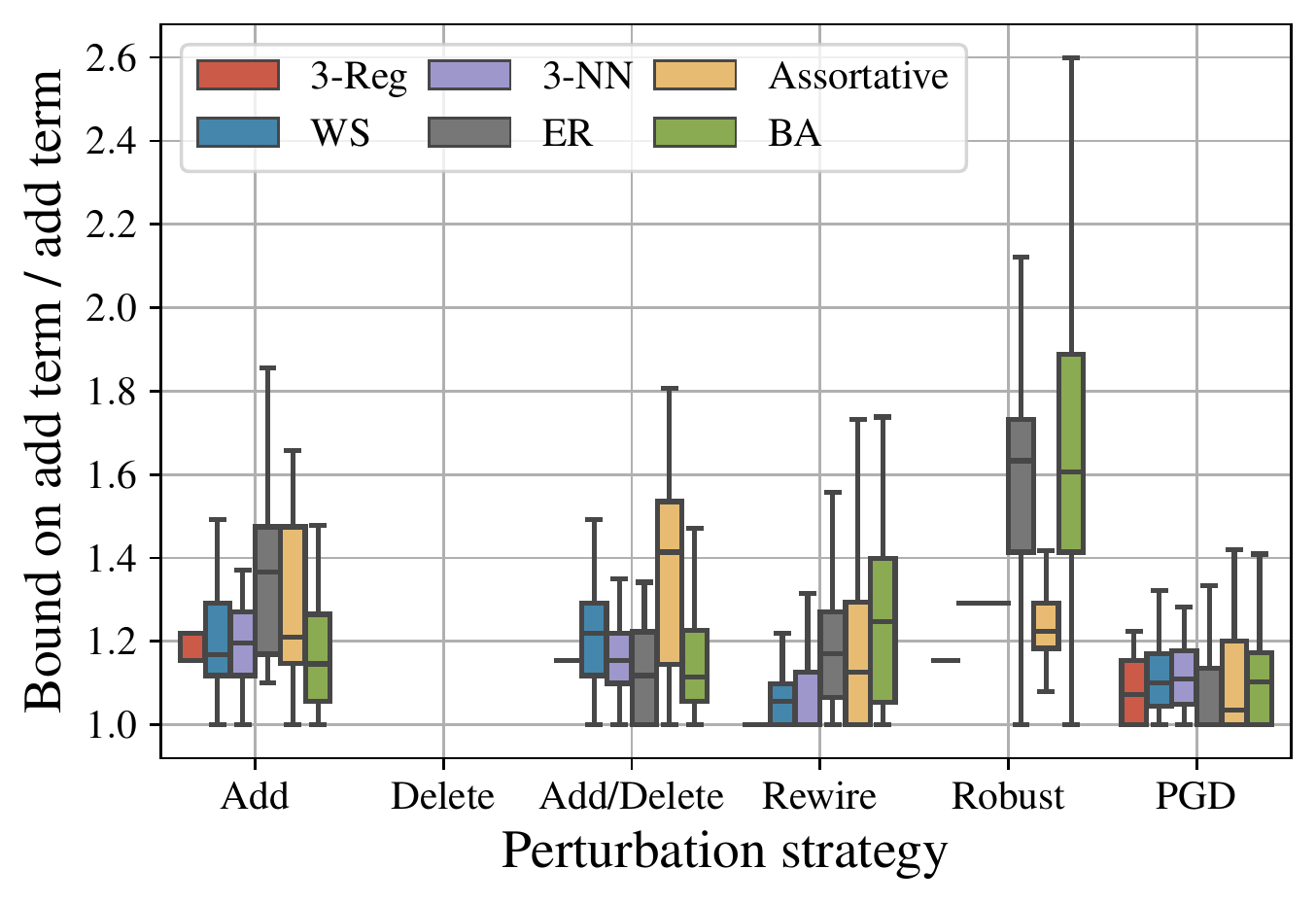}
    \end{subfigure} 
    \begin{subfigure}[b]{0.3\textwidth}
        \centering
        \includegraphics[width=\textwidth]{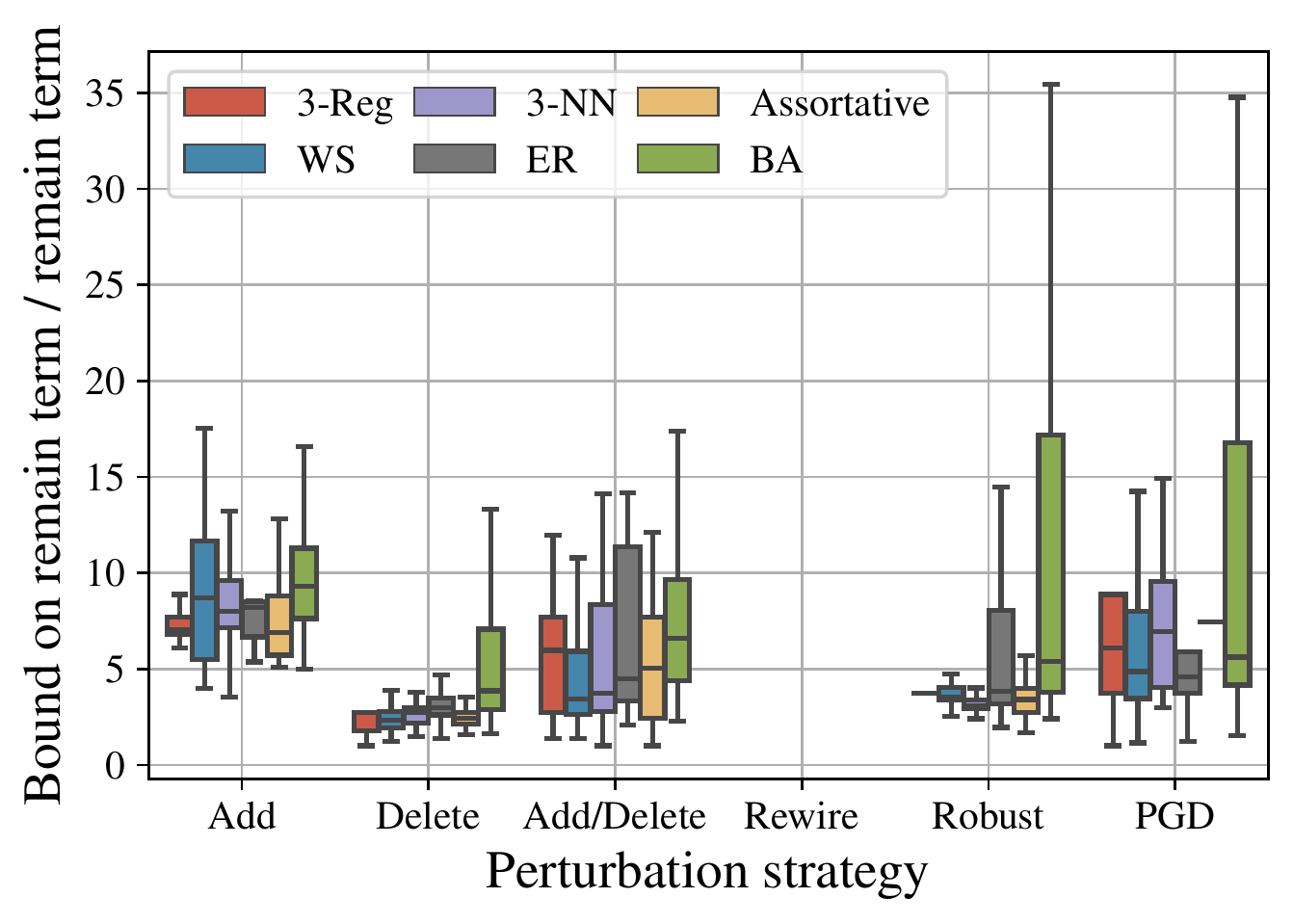}
    \end{subfigure}
    \caption{Looseness of the bound for the three terms in Eq.~(\ref{eq:Eu}). The first term, as well as the bound for the first term, evaluate to zero for the Add strategy and thus looseness is undefined. The same applies to the second term for the Delete strategy, and the third term for the Rewire strategy.}
    \label{fig:bound_all_terms}
\end{figure*}

\end{document}

%% file: figures/rewire-before.tex
\begin{tikzpicture}
\Vertex[x=0,y=0, label=$u$]{u}
\Vertex[x=0,y=1, label=$v$]{v}
\Vertex[x=1,y=0, label=$u'$]{uprime}
\Vertex[x=1,y=1, label=$v'$]{vprime}
\Edge[RGB,color={162.56 ,  19.968,  47.104}](u)(v)
\Edge[RGB,color={162.56 ,  19.968,  47.104}](uprime)(vprime)
\Vertex[x=-1,y=1.5, Pseudo]{dummy13}
\Vertex[x=-1,y=1, Pseudo]{dummy12}
\Vertex[x=-1,y=0, Pseudo]{dummy11}
\Edge[bend=25, opacity=0.2](u)(dummy11)
\Edge[bend=15, opacity=0.2](u)(dummy12)
\Edge[bend=-15, opacity=0.2](v)(dummy12)
\Edge[bend=-25, opacity=0.2](v)(dummy13)
\Vertex[x=2,y=1.5, Pseudo]{dummy23}
\Vertex[x=2,y=1, Pseudo]{dummy22}
\Vertex[x=2,y=0, Pseudo]{dummy21}
\Edge[bend=25, opacity=0.2](uprime)(dummy21)
\Edge[bend=-15, opacity=0.2](vprime)(dummy22)
\Edge[bend=25, opacity=0.2](vprime)(dummy23)
\end{tikzpicture}

%% file: figures/rewire-after.tex
\begin{tikzpicture}
\Vertex[x=0,y=0, label=$u$]{u}
\Vertex[x=0,y=1, label=$v$]{v}
\Vertex[x=1,y=0, label=$u'$]{uprime}
\Vertex[x=1,y=1, label=$v'$]{vprime}
\Edge[RGB,color={ 0.   , 114.432, 189.696}](u)(uprime)
\Edge[RGB,color={ 0.   , 114.432, 189.696}](v)(vprime)
\Vertex[x=-1,y=1.5, Pseudo]{dummy13}
\Vertex[x=-1,y=1, Pseudo]{dummy12}
\Vertex[x=-1,y=0, Pseudo]{dummy11}
\Edge[bend=25, opacity=0.2](u)(dummy11)
\Edge[bend=15, opacity=0.2](u)(dummy12)
\Edge[bend=-15, opacity=0.2](v)(dummy12)
\Edge[bend=-25, opacity=0.2](v)(dummy13)
\Vertex[x=2,y=1.5, Pseudo]{dummy23}
\Vertex[x=2,y=1, Pseudo]{dummy22}
\Vertex[x=2,y=0, Pseudo]{dummy21}
\Edge[bend=25, opacity=0.2](uprime)(dummy21)
\Edge[bend=-15, opacity=0.2](vprime)(dummy22)
\Edge[bend=25, opacity=0.2](vprime)(dummy23)
\end{tikzpicture}